\documentclass[10pt,journal,compsoc]{IEEEtran}
\pdfoutput=1
\ifCLASSOPTIONcompsoc
\usepackage[nocompress]{cite}
\else
  \usepackage{cite}
\fi
\usepackage{algorithm}
\usepackage{algpseudocode}
\usepackage{epsfig}
\usepackage{graphicx}
\usepackage{amsfonts}
\usepackage{comment}
\usepackage{caption}
\usepackage{color}
\usepackage{amssymb}
\usepackage{amsmath}
\usepackage{amsthm}
\usepackage{multirow}
\usepackage{enumitem}
\usepackage{etoolbox}
\usepackage{balance}
\newtheorem{lemma}{Lemma}
\AfterEndEnvironment{lemma}{\noindent\ignorespaces}

\newtheorem{definition}{Definition}[section]
\AfterEndEnvironment{definition}{\noindent\ignorespaces}
\newtheorem{example}{Example}[section]
\AfterEndEnvironment{example}{\noindent\ignorespaces}

\ifCLASSINFOpdf
\else
\fi
\begin{document}

\title{A New Methodology for Search Space Reduction in QoS Aware Semantic Web Service Composition}
% QoS Aware Semantic Web Service Composition Using Abstraction and Refinement
\author{Soumi~Chattopadhyay,~\IEEEmembership{Member,~IEEE,} and
        Ansuman~Banerjee,~\IEEEmembership{Member,~IEEE}

\IEEEcompsocitemizethanks{\IEEEcompsocthanksitem Soumi Chattopadhyay is now in Indian Institute of Information Technology, Guwahati $\quad$ E-mail: $soumi@iiitg.ac.in$ and Ansuman Banerjee is in Indian Statistical Institute, Kolkata, India. $\quad$ E-mail: $ansuman@isical.ac.in$.

This  work  has  been  submitted  to  the  IEEE  for  possible  publication. Copyright may be transferred without notice, after which this version mayno longer be accessible
}}

\makeatletter
\long\def\@IEEEtitleabstractindextextbox#1{\parbox{0.922\textwidth}{#1}}
\makeatother

% \maketitle
\IEEEtitleabstractindextext{
\begin{abstract}
The semantic information regulates the expressiveness of a web service. State-of-the-art approaches in web services research have used the semantics of a web service for different purposes, mainly for service discovery, composition, execution etc. In this paper, our main focus is on semantic driven Quality of Service (QoS) aware service composition. Most of the contemporary approaches on service composition have used the semantic information to combine the services appropriately to generate the composition solution. However, in this paper, our intention is to use the semantic information to expedite the service composition algorithm. Here, we present a service composition framework that uses semantic information of a web service to generate different clusters, where the services are semantically related  within a cluster. Our final aim is to construct a composition solution using these clusters that can efficiently scale to large service spaces, while ensuring solution quality. Experimental results show the efficiency of our proposed method.
\end{abstract}
}

\maketitle

\IEEEdisplaynontitleabstractindextext
\IEEEpeerreviewmaketitle
\ifCLASSOPTIONcaptionsoff
  \newpage
\fi

\section{Introduction}\label{sec:intro}
\noindent
With proliferation of network technologies, services computing has attracted a significant research attention. Web services have become a useful tool to facilitate many necessary activities throughout the world. A web service performs a specific task to satisfy a user requirement. The performance of a web service is measured by a set of QoS parameters like response time, throughput, reliability, availability etc. Sometimes, a single web service is sufficient to meet a user requirement, sometimes, multiple web services are required to do so. In the latter case, web service composition \cite{7226855} is required. Though optimality \cite{schuller2012cost,rodriguez2015hybrid} is a primary concern in composition, it comes at the cost of more computation time, which is another key issue in recent times. With increase in the scale of web services \cite{7027801}, scalable service composition \cite{chattopadhyay2015scalable,lecue2009towards} is becoming a challenge. 

%%%%%%%%%%%%%%%%%%%%%%%%%%%%%%%%%%%%%%%%%%%%%%%%%%%%%%%%%%%%%%%%%%%%%%%%%%%%%%%%%%%%%%%%%%%%%%%%%%%%%
\begin{table*}
\centering
\caption{DataSet $D$: Service Repository}
\label{tab:ourData}
%  \tbl{DataSet $D$: Service Repository\label{tab:ourData}}{
 \small
\begin{tabular}{l|l|l}
\hline\hline
 & \multicolumn{1}{c|}{\bf 1} & \multicolumn{1}{c}{\bf 2}\\\hline\hline
 {\em\bf Method}: & BWImageToBarcode & BinImageToBarcode\\ \hline
 {\em\bf Service Name}: & ${\mathcal{S}}_1$, ${\mathcal{S}}_2$ & ${\mathcal{S}}_3$, ${\mathcal{S}}_4$\\\hline
 {\em\bf Input}: & BWImage & BinImage\\ \hline
 {\em\bf Precondition}: &BWImage must be in jpeg format & BinImage must be in jpeg format\\
                        &BWImage size must be less than $50KB$ & BinImage size must be less than $50KB$\\ \hline
 {\em\bf Output}: & Barcode &  Barcode\\ \hline
 {\em\bf Postcondition}:& - & -\\
 
 \hline\hline
 & \multicolumn{1}{c|}{\bf 3} & \multicolumn{1}{c}{\bf 4}\\\hline\hline
 {\em\bf Method}:& ImageToBarcode & AnyImageToBarcode\\ \hline
 {\em\bf Service Name}: & ${\mathcal{S}}_5$, ${\mathcal{S}}_6$& ${\mathcal{S}}_7$, ${\mathcal{S}}_8$, ${\mathcal{S}}_9$ \\\hline
 {\em\bf Input}: & Image&Image \\ \hline
 {\em\bf Precondition}: & Image must be in jpeg format& Image size must be less than $50KB$ \\
                        & Image size must be less than $50KB$&\\ \hline
 {\em\bf Output}: & Barcode& Barcode \\ \hline
 {\em\bf Postcondition}:& -& - \\

 \hline\hline
 &\multicolumn{1}{c|}{\bf 5} & \multicolumn{1}{c}{\bf 6}\\\hline\hline
 {\em\bf Method}: & ImageToIDConverter & EuropeanBarcodeToProductID \\ \hline
 {\em\bf Service Name}: & ${\mathcal{S}}_{10}$& ${\mathcal{S}}_{11}$, ${\mathcal{S}}_{12}$\\
\hline
 {\em\bf Input}: & Image& EANBC  \\ \hline
 {\em\bf Precondition}: & -& - \\ \hline
 {\em\bf Output}: & ProductID &  ProductNumber  \\ \hline
 {\em\bf Postcondition}:& - & -  \\
 
 \hline\hline
 & \multicolumn{1}{c|}{\bf 7} & \multicolumn{1}{c}{\bf 8}\\\hline\hline
 {\em\bf Method}:& BarcodeToProductID & BarcodeToEncoded9ProductID\\\hline
 {\em\bf Service Name}: & ${\mathcal{S}}_{13}$ & ${\mathcal{S}}_{14}$, ${\mathcal{S}}_{15}$, ${\mathcal{S}}_{16}$\\\hline
 {\em\bf Input}: & Barcode & Barcode\\\hline
 {\em\bf Precondition}: & -& - \\\hline
 {\em\bf Output}: & PID & $PID_9$ \\\hline
 {\em\bf Postcondition}:& -& - \\

 \hline\hline
 & \multicolumn{1}{c|}{\bf 9} & \multicolumn{1}{c}{\bf 10}\\\hline\hline
 {\em\bf Method}:& BarcodeToEncoded14ProductID &  GetReviewInEnglish \\\hline
 {\em\bf Service Name}: & ${\mathcal{S}}_{17}$ &  ${\mathcal{S}}_{18}$, ${\mathcal{S}}_{19}$, ${\mathcal{S}}_{20}$ \\\hline
 {\em\bf Input}: & Barcode &   ProductID \\\hline
 {\em\bf Precondition}: & - &  - \\\hline
 {\em\bf Output}: & $PID_{14}$ &   Review \\\hline
 {\em\bf Postcondition}: & - & Review written in English \\ 

 \hline\hline
  & \multicolumn{1}{c|}{\bf 11}& \multicolumn{1}{c}{\bf 12} \\\hline\hline
 {\em\bf Method}:  & GetReview & GetReviewAndRatingInEnglish \\ \hline
 {\em\bf Service Name}: & ${\mathcal{S}}_{21}$, ${\mathcal{S}}_{22}$, ${\mathcal{S}}_{23}$ & ${\mathcal{S}}_{24}$, ${\mathcal{S}}_{25}$ \\\hline
 {\em\bf Input}:  & ProductID & $PID_9$\\ \hline
 {\em\bf Precondition}:  & - & $PID_9$ in alpha-numeric value\\ \hline
 {\em\bf Output}: & Review & Review, Rating\\ \hline
 {\em\bf Postcondition}:& Review written in local language & Review written in English\\

 \hline\hline
  & \multicolumn{1}{c|}{\bf 13}& \multicolumn{1}{c}{\bf 14}\\\hline\hline
 {\em\bf Method}:  & GetReviewAndRating& GetCheapShopAtPos\\ \hline
 {\em\bf Service Name}: & ${\mathcal{S}}_{26}$& ${\mathcal{S}}_{27}$, ${\mathcal{S}}_{28}$, ${\mathcal{S}}_{29}$, ${\mathcal{S}}_{30}$\\\hline
 {\em\bf Input}:  & $PID_9$&  ProductID, GPS\\ \hline
 {\em\bf Precondition}:  & $PID_9$ in alpha-numeric value& -\\ \hline
 {\em\bf Output}:  & Review, Rating&   Price, GPS\\ \hline
 {\em\bf Postcondition}: & Review written in local language & GPS must be in the radius of 100 meter\\
 && Cheapest price in 100 meter radius\\
 
 \hline\hline
  & \multicolumn{1}{c}{\bf 15}\\\hline\hline
 {\em\bf Method}:  & GetShopAtPos\\ \hline
 {\em\bf Service Name}: & ${\mathcal{S}}_{31}$, ${\mathcal{S}}_{32}$\\
\hline
 {\em\bf Input}:  & PID, GPS\\ \hline
 {\em\bf Precondition}:  & -\\ \hline
 {\em\bf Output}:  & Price, GPS\\ \hline
 {\em\bf Postcondition}: & GPS of the shop is the nearest one\\
 
 \hline\hline
  \multicolumn{3}{c}{\bf Ontology}\\\hline\hline
  \multicolumn{3}{c}{}\\
  \multicolumn{3}{c}{Concept: PID, BinImage; $f(ProductID) = f(PID) = PID;$ $PID_9$ $f(BWImage) = BinaryImage;$}\\ 
  \multicolumn{3}{c}{$f(BinImage) = BinaryImage;$ $\text{EANBC} \sqsubset \text{Barcode};\text{BinaryImage} \sqsubset \text{Image};$
 $PID_{9} \sqsubset PID;$ $PID_{14} \sqsubset PID$}\\
\hline
\end{tabular}
% }
\end{table*}
%%%%%%%%%%%%%%%%%%%%%%%%%%%%%%%%%%%%%%%%%%%%%%%%%%%%%%%%%%%%%%%%%%%%%%%%%%%%%%%%%%

A significant body of research focuses on finding optimal QoS aware composition solutions \cite{papadias2003optimal,6914562}. However, these methods often fail to scale for large service repositories in real time. Therefore, a number of heuristic methods \cite{6009375,1357986,DBLPChattopadhyayB16} have been proposed in literature. These methods also have their own limitations. Many of these fail to deliver the solution with desired quality, since they do not navigate the entire search space. The trade-off between optimality and efficiency has been explored extensively in web service literature \cite{6009422,alrifai2009combining,5552793}.
In this work, we handle the trade-off between optimality and efficiency from a different perspective.
In a realistic setting, a number of services are often deployed to fulfill the same purpose. Though these services are intended to serve similar functionality, they may have syntactically different interfaces. Sometimes, identifying two semantically related services becomes extremely important to compute the composition solution efficiently. The aim of this work is to identify semantically related services obtained from an ontology \cite{wagner2011qos} stored in a service repository along with the web service set. Using this information, we further aim to form service groups essentially required for generating efficient composition solution. The crux of our proposal is to reduce the search space of the underlying composition algorithm by clustering the semantically related services and thus, expedite its solution computation time. Once we cluster the services into multiple groups, each service group is then represented by a new service. Each new service is then assigned a set of QoS values that are the best representative in its respective group. The composition algorithm is first carried out in the new service space to find a solution satisfying all QoS constraints. Since the algorithm explores a smaller search space, it becomes more efficient than the original one. However, some information regarding service QoS parameters may be lost due to this clustering. Therefore, the composition algorithm, performed with the new representative service, may fail in producing a solution to a query satisfying all the constraints. In this paper, this clustering is termed as an {\em{abstraction}}. To deal with constraint violation issue, we further propose to refine an abstraction, which ensures a solution satisfying all the necessary constraints if one exists originally. Our framework is, therefore, sound and complete. 

% The semantic information of each service and their input-output parameters are stored in an ontology \cite{wagner2011qos}, which is maintained in a service repository along with the set of web services.

{\bf{Novelty of our work:}} Our method is built on the foundation of abstraction refinement, a standard technique for handling large state spaces, originally proposed in \cite{clarke2000counterexample} in the context of formal verification. Abstraction refinement based on syntactic similarity in terms of input-output parameters has been dealt with in our earlier work \cite{7933195}. However, the notion of semantics of a web service has not been dealt with in that work. In this paper, we bring the notion of semantics and propose a methodology that provides a scalable way of pruning the composition search space. Our method can be applied on top of any service composition algorithm to improve its performance. In \cite{wagner2011qos}, the authors proposed functional clustering based on semantics of the services. However, the focus of \cite{wagner2011qos} is on reducing the search space with the objective of improving the reliability and flexibility, while our focus is to reduce the search space and thus expedite solution composition time. Moreover, the clustering methods that we consider is different from the one in \cite{wagner2011qos}.

{\bf{Our Contribution:}} In this work, we have the following contributions:
\begin{itemize}
 \item We first propose different abstractions based on the semantics of the services in the service repository, that help to reduce the search space and thus to expedite solution computation time.
 \item In order to ensure that our algorithm always produce a valid solution (i.e., a solution satisfying all the constraints), we further propose two refinements corresponding to each solution.
 \item We have performed extensive experiments and it is evident from our experiments that on an average our framework is more efficient, when compared with the performance of the underlying composition algorithm without abstraction.
\end{itemize}

This paper is organized as follows. Section \ref{sec:preliminaries} presents some preliminary concepts. Section \ref{sec:method} demonstrates our proposed methodologies. Section~\ref{sec:result} shows the results obtained using our framework. Section~\ref{related} includes a discussion on related literature, while Section \ref{sec:conclusion} concludes our paper.

\section{Working Example}\label{sec:example}
\noindent
We illustrate our proposal on the following running example as shown in Table \ref{tab:ourData} containing 32 services ${\cal{S}}_{1}$ to ${\cal{S}}_{32}$. 
We consider the following query:
A user provides a binary image of the Barcode of of a product in jpeg format with size less than $50KB$ along with his current location to obtain the review of the product in English and the location coordinates GPS of the nearest shop and the product price. The query $\cal{Q}$ is formally written as:
 ${\cal{Q}} = (\{BinaryImage, GPS\}, \{Review, GPS, Price\}$, \{Image in jpeg format, Image size less than $50KB$\}, \{Review written in English\}, \{\}, \{\}).

\section{Overview and Problem Formulation}\label{sec:preliminaries}
\noindent
We illustrate our proposal on the following running example as shown in Table \ref{tab:ourData} containing 32 services ${\mathcal{S}}_{1}$ to ${\mathcal{S}}_{32}$. 
We consider the following query:
A user provides a binary image of the Barcode of of a product in jpeg format with size less than $50KB$ along with his current location to obtain the review of the product in English and the location coordinates GPS of the nearest shop and the product price. The query $\mathcal{Q}$ is formally written as:
 ${\mathcal{Q}} = (\{BinaryImage, GPS\}, \{Review, GPS, Price\}$, \{Image in jpeg format, Image size less than $50KB$\}, \{Review written in English\}, \{\}, \{\}).
We begin with defining a few terminologies.

% In this section, we discuss some background concepts needed to build the foundation of our work. 
% We illustrate our proposal on a running example as shown in Table \ref{tab:ourData}. 

%  A web service is a software component that performs a specific task and is characterized by a set of functional and non-functional parameters. The functional parameters include a set of input and output parameters, whereas, the non-functional parameters refer to a set of QoS parameters. A web service is formally defined as:

% that are required to evaluate the performance of a web service

\begin{definition}{\bf[Web Service]}:{\em~
  A web service ${\mathcal{S}}_i$ is a 6-tuple { $({\mathcal{S}}^{ip}_i, {\mathcal{S}}^{op}_i, {\mathcal{M}}_i, {\mathcal{P}}^i, {\mathcal{PR}}_i, {\mathcal{PO}}_i)$}, where, ${\mathcal{S}}^{ip}_i$ and ${\mathcal{S}}^{op}_i$ are the set of input and output parameters of ${\mathcal{S}}_i$ respectively. ${\mathcal{M}}_i$ is the method implemented by ${\mathcal{S}}_i$. { ${\mathcal{P}}^i = ({\mathcal{P}}^i_1, {\mathcal{P}}^i_2, \ldots, {\mathcal{P}}^i_m)$} are the values of the set of QoS parameters. ${\mathcal{PR}}_i$ represents a set of preconditions, defined over ${\mathcal{S}}^{ip}_i$ and specified in quantifier free first order logic \cite{lindstrom1966first}. ${\mathcal{PO}}_i$ represents a set of postconditions, defined over the set of output parameters of ${\mathcal{S}}_i$ and specified in quantifier free first order logic.
   \hfill$\blacksquare$
}\end{definition}

\noindent
A web service is executed, once all its inputs are available and the set of preconditions are satisfied. If a web service is executed, it produces the corresponding set of outputs, while ensuring the set of postconditions. The pre and post conditions are optional for a service. Therefore, the pre or post condition / both of them can be null for a service.

\begin{example}{\em~
Example services can be found in Table \ref{tab:ourData}. The table contains 32 services ${\mathcal{S}}_{1}$ to ${\mathcal{S}}_{32}$. For each service, a brief description of the method it implemented is provided in the table.
\hfill$\blacksquare$}
\end{example}

The QoS parameters of a web service are classified into two categories: (a) {\em Positive QoS}, for which a high value is desirable (e.g., reliability) (b) {\em Negative QoS}, for which a low value is desirable (e.g., response time). 

Two input / output parameters may have syntactically different names. However, both of them may be semantically related. An ontology, which is stored in the service repository along with the web services, maintains all such information. To take into account the semantics of the parameters, a list of concepts and a mapping from each input / output parameter to a concept are maintained in the ontology. 

\begin{definition}{\bf[Concept]}:{\em~
 Each input / output parameter corresponding to an ontology maps to a unique entity ${\mathcal{C}}_i$ of the ontology, called a concept. 
 \hfill$\blacksquare$}
\end{definition}

\begin{definition}{\bf[Parameter mapping]}:{\em~
 Parameter mapping $f$ is a mapping of elements in $IO$ to $C$, where ${\mathcal{IO}}$ and ${\mathcal{C}}$ are respectively the set of input-output parameters and the set of concepts corresponding to an ontology.
 \hfill$\blacksquare$}
\end{definition}

\begin{example}{\em~
Consider services ${\mathcal{S}}_1$ and ${\mathcal{S}}_3$ in Table \ref{tab:ourData}. ${\mathcal{S}}_1$ takes a BWImage as input, while ${\mathcal{S}}_3$ takes a BinImage as input. According to the ontology, both refer to the same concept binaryImage. Therefore, $f(binImage) = f(BWImage) = binaryImage$.}
 \hfill$\blacksquare$
\end{example}

\noindent
Two concepts ${\mathcal{C}}_1$ and ${\mathcal{C}}_2$ in an ontology are related in either of the following ways: (a) ${\mathcal{C}}_1 = {\mathcal{C}}_2$: ${\mathcal{C}}_1$ and ${\mathcal{C}}_2$ are identical concepts. (b) ${\mathcal{C}}_1 \sqsubset {\mathcal{C}}_2$: ${\mathcal{C}}_1$ is a sub concept of ${\mathcal{C}}_2$, while ${\mathcal{C}}_2$ is a super concept of ${\mathcal{C}}_1$. (c) ${\mathcal{C}}_1 \neq {\mathcal{C}}_2$: ${\mathcal{C}}_1$ and ${\mathcal{C}}_2$ are not related.

\begin{example}{\em~
 Consider 4 different concepts: binaryImage, grayImage, RGBImage, Image. The relationship mentioned in an ontology, is as follows: $binaryImage \sqsubset grayImage \sqsubset Image$; $grayImage \ne RGBImage$; $RGBImage \sqsubset Image$;
 \hfill$\blacksquare$}
\end{example}

\begin{definition}{\bf[Query]}:{\em~
 A query ${\mathcal{Q}}$ is a 6-tuple {$({\mathcal{Q}}^{ip}, {\mathcal{Q}}^{op}, {\mathcal{IS}}, {\mathcal{OR}}, {\mathcal{OP}}, {\mathcal{CP}})$}, where, {${\mathcal{Q}}^{ip}$} and {${\mathcal{Q}}^{op}$} are the set of given input and desired output parameters respectively. {${\mathcal{IS}}$} refers to the input specification, defined over {${\mathcal{Q}}^{ip}$} and specified in quantifier free first order logic. {${\mathcal{OR}}$} represents the output requirement, defined over {${\mathcal{Q}}^{op}$} and specified in quantifier free first order logic. {${\mathcal{CP}}$} represents a set of QoS constraints, typically, the bounds on the worst case value of the QoS parameters. {${\mathcal{OP}}$} refers to the objectives of the query (e.g., maximizing throughput, minimizing response time).
  \hfill$\blacksquare$}
\end{definition}

\begin{example}\label{example:query}{\em~
 Consider a query. A user provides a binary image of the Barcode of a product in jpeg format with size less than $50KB$ to obtain the review of the product in English. The query $\mathcal{Q}$ is formally written as:
 ${\mathcal{Q}} = (\{BinaryImage\}, \{Review\}$, \{Image in jpeg format, Image size less than $50KB$\}, \{Review written in English\}, \{Maximum response time is 500ms\}, \{Maximize the reliability\}). The last two elements of the tuple represent QoS constraints and objective of the query respectively. We use this query throughout this paper.
 \hfill$\blacksquare$}
\end{example}
% \begin{example}\label{example:queryem}{\em~
% Consider a query. A user provides a binary image of the Barcode of a product in jpeg format with size less than $50KB$ along with his current location to obtain the review of the product in English and GPS of the nearest shop and the product price. The query $\mathcal{Q}$ is formally written as:
% ${\mathcal{Q}} = (\{BinaryImage, GPS\}, \{Review, GPS, Price\}$, \{Image in jpeg format, Image size less than $50KB$\}, \{Review written in English\}, \{Maximum response time is 500ms\}, \{Maximize the reliability\}). The last two elements of the tuple represent QoS constraints and objective of the query respectively. We use this query throughout this paper.
% \hfill$\blacksquare$}
% \end{example}

\noindent
Once a query comes to the system, a set of services are activated based on the query inputs and input specification. We now define the notion of service activation.

\begin{definition}{\bf[Service Activation]}:{\em~
 A service ${\mathcal{S}}_i$ is activated, while serving a query ${\mathcal{Q}}$, if the following are satisfied: (a) All the inputs of ${\mathcal{S}}_i$ are available (either query inputs or the set of outputs of already activated services) in the system, i.e.,
 { $\forall i \in {\mathcal{S}}^{ip}_i, \exists i' \in {\mathcal{Q}}^{ip} \cup {\mathcal{S}}^{op}_{AC}$; such that $f(i') \sqsubseteq f(i)$}. (b) All preconditions of ${\mathcal{S}}_i$ defined over its set of inputs are also satisfied either by the query input specification or by the postconditions of the services that are activated by the query, i.e.,
 { ${\mathcal{IS}} \vee (\bigvee \limits_{\forall {\mathcal{S}}_j \in {\mathcal{S}}_{AC}} {\mathcal{PO}}_j) \rightarrow {\mathcal{PR}}_i$},
 where, ${\mathcal{S}}_{AC}$ is the set of services activated by ${\mathcal{Q}}^{ip}$ directly / transitively and ${\mathcal{S}}^{op}_{AC}$ is the union of the outputs of the services in ${\mathcal{S}}_{AC}$.
 \hfill$\blacksquare$}
\end{definition}

 \begin{example}{\em~
 Consider two services, as defined below
  \begin{table}\centering\scriptsize
 \begin{tabular}{l|l|l}
 \hline
  {\em\bf Service Name}: & \multicolumn{1}{c|}{${\mathcal{S}}_1$}  	& \multicolumn{1}{c}{${\mathcal{S}}_2$}\\\hline
  {\em\bf Input}:        & Image                  		& ProductName \\ 
  {\em\bf Precondition}: & Image in jpeg / png format      	& ProductName in English\\                        
  {\em\bf Method}:       & GetImageToProduct      		& GetReview\\ 
  {\em\bf Output}:       & ProductName            		& Review\\ 
  {\em\bf Postcondition}:& ProductName in English 		& Review in local language\\
 \hline
 \end{tabular}
 \end{table}
 and a query ${\mathcal{Q}} = (\{BWImage\},$ $\{Review\}$, \{BWImage in jpeg format\}, \{Review in local language\}, \{\}, \{\}). According to the ontology, $BWImage$ is a sub concept of $Image$. Therefore, $BWImage$ can be fed as an input to ${\mathcal{S}}_1$. Moreover, according to the input specification, the BWImage is in jpeg format. Therefore, with the given input specification, the precondition of ${\mathcal{S}}_1$ is also satisfied. Hence, ${\mathcal{S}}_1$ is activated first. Once ${\mathcal{S}}_1$ is activated, from the output of ${\mathcal{S}}_1$, ${\mathcal{S}}_2$ is also activated, since ${\mathcal{PO}}_1$ satisfies ${\mathcal{PR}}_2$. Finally, ${\mathcal{S}}_2$ provides ${\mathcal{Q}}^{op}$ while ensuring the output requirement.
  \hfill$\blacksquare$}
 \end{example}

\noindent
It may be noted, in this paper, we use the symbol $A \rightarrow B$ to denote $B$ is weaker than $A$ for any two quantifier free first order logic formulas $A$ and $B$. In other words, whenever $A$ holds $B$ holds as well, however, the converse is not true. Once a query comes to the system, the services are activated directly / transitively. Gradually, a dependency graph \cite{chattopadhyay2015scalable} is constructed to serve the query. We now show an example of a dependency graph on our running example.

\begin{example}{\em~
Considering our running example (shown in Table \ref{tab:ourData}), the dependency graph, constructed in response to the query ${\mathcal{Q}}$ (shown in Example \ref{example:query}) is shown in Figure \ref{fig:exampleLevel0}. Each rectangle presents a group of services having same inputs, outputs, preconditions and postconditions. Each number, situated at the left side of the rectangle, presents the number of services in that group.
\begin{figure}[!htb]
\centering
\includegraphics[width=\linewidth]{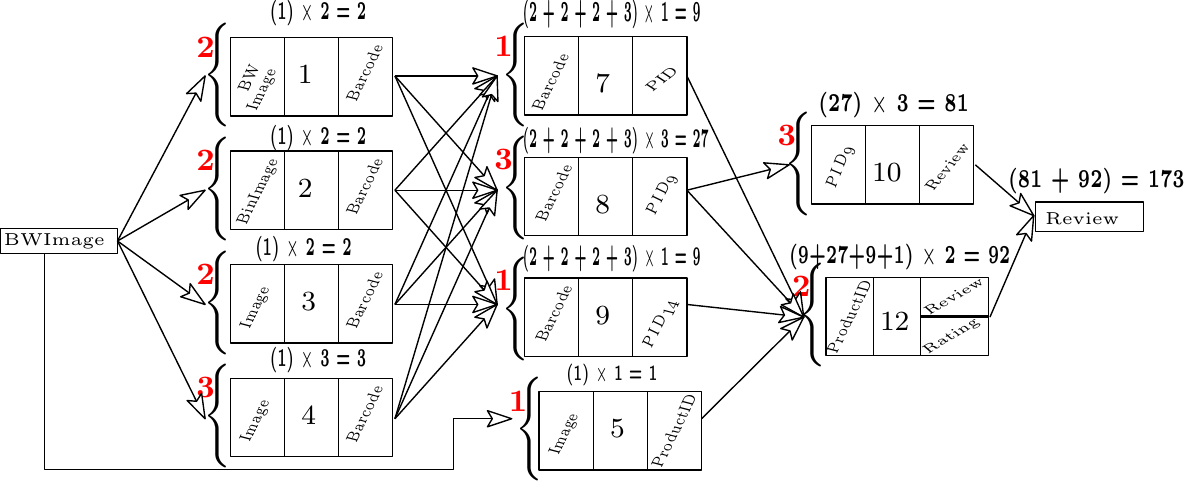}
\caption{Dependency Graph to respond ${\mathcal{Q}}$}
\label{fig:exampleLevel0}
\end{figure}
 \noindent
 It may be noted, initially, based on the query inputs and the input specifications, ${\mathcal{S}}_1$ to ${\mathcal{S}}_{10}$ are activated. From the outputs of ${\mathcal{S}}_1$ to ${\mathcal{S}}_{10}$ and their postconditions, ${\mathcal{S}}_{13}$ to ${\mathcal{S}}_{17}$ are activated. Finally, ${\mathcal{S}}_{18}$ to ${\mathcal{S}}_{20}$, ${\mathcal{S}}_{24}$ and ${\mathcal{S}}_{25}$ are activated, that in turn produces the outputs that serve the query and ensure the output requirements. As shown in Figure \ref{fig:exampleLevel0}, the dependency graph without abstraction consists of 20 services {$(2+2+2+3+1+3+1+1+3+2)$}(rest are not activated). The total number of solutions (in terms of functional dependencies) is 173. The calculation is shown in the figure. Above each rectangle, the calculation shows the number of ways the outputs corresponding to the rectangle can be produced. Finally, the calculation above the query output shows the number of solutions.
% {
% ($|((\{{\mathcal{S}}_1, {\mathcal{S}}_2, \ldots, {\mathcal{S}}_9\} \times \{{\mathcal{S}}_{13}, {\mathcal{S}}_{14}, \ldots, {\mathcal{S}}_{17}\}) \cup \{{\mathcal{S}}_{10}\})  \times \{{\mathcal{S}}_{18}, {\mathcal{S}}_{19}, {\mathcal{S}}_{20}, {\mathcal{S}}_{24}, {\mathcal{S}}_{25}\}| = $
% $(9 \times 5 + 1) \times 5 = 230$.}).
% It may be noted, each member of the above set is a solution to the query.
 \hfill$\blacksquare$}
\end{example}

\noindent
We now formally discuss the service composition problem. 

\noindent
{\bf\em{Problem formulation}}: We are given the following: 
\begin{enumerate}
 \item A set of web services ${\mathcal{S}} = \{{\mathcal{S}}_1, {\mathcal{S}}_2, \ldots, {\mathcal{S}}_n\}$.
 \item A set of concepts ${\mathcal{C}}$.
 \item For each input-output parameter $io$, the mapping to its concept $f(io)$.
 \item A query ${\mathcal{Q}}$.
\end{enumerate}

\noindent
The aim of service composition is to find a set of services to {\em{serve the query}} with appropriate objectives as specified therein. In this paper, our aim is to expedite the process of composition by reducing the search space as explained in the next section. 

% As discussed earlier, we do not propose any new service composition algorithm in this paper. Our framework can be seamlessly integrated on top of any service composition algorithm and our motivation is to expedite the performance of the composition algorithm.

% During service composition, a service dependency graph is constructed. For a large service repository, the size of the dependency graph can be significantly large to handle it efficiently. 

\section{Detailed Methodology}\label{sec:method}
\noindent
Our framework comprises two key steps: {\em abstraction} done at preprocessing time and {\em refinement} performed during execution time, i.e., when a query arrives into the system. In the preprocessing phase, a set of services are abstracted based on their functional semantics to obtain a new abstract service. The major advantage of this step is that it reduces the search space and therefore, expedites computation time. Once a query comes to the system, service composition is performed on the abstract service space first. However, the abstraction has a major limitation due to the lack of information of all the services. In some cases, the solution to a query obtained from the abstract service space may not satisfy all the QoS constraints. Therefore, we further propose to refine an abstraction. Refinement is done in the execution phase. The main benefit of this step is that it makes the method complete \cite{7933195}. However, an additional computational overhead may be added during composition, if refinement is applied. In the following subsections, we discuss 3 abstraction methods, each followed by the corresponding refinement methods. 

\subsection{{\bf \em First level Abstraction: Based on Equivalence Relation}}
\noindent
We first define the notion of equivalence between services.

\begin{definition}{[\bf Service Equivalence:]}\label{def:Equi}{
 Two services ${\mathcal{S}}_1, {\mathcal{S}}_2$ are functionally equivalent, expressed as ${\mathcal{S}}_1 \simeq {\mathcal{S}}_2$, if the following conditions are satisfied:
 \begin{enumerate}
  \item $\forall i \in {\mathcal{S}}^{ip}_1, \exists i' \in {\mathcal{S}}^{ip}_2, \text{ such that, } f(i) = f(i').$
  \item $\forall i' \in {\mathcal{S}}^{ip}_2, \exists i \in {\mathcal{S}}^{ip}_1, \text{ such that, } f(i) = f(i').$
  \item ${\mathcal{PR}}_1 \leftrightarrow {\mathcal{PR}}_2, i.e., {\mathcal{PR}}_1 \rightarrow {\mathcal{PR}}_2 \text{ and } {\mathcal{PR}}_2 \rightarrow {\mathcal{PR}}_1.$
  \item $\forall o \in {\mathcal{S}}^{op}_1, \exists o' \in {\mathcal{S}}^{op}_2, \text{ such that, } f(o) = f(o').$
  \item $\forall o' \in {\mathcal{S}}^{op}_2, \exists o \in {\mathcal{S}}^{op}_1, \text{ such that, } f(o) = f(o').$
  \item ${\mathcal{PO}}_1 \leftrightarrow {\mathcal{PO}}_2, i.e., {\mathcal{PO}}_1 \rightarrow {\mathcal{PO}}_2 \text{ and } {\mathcal{PO}}_2 \rightarrow {\mathcal{PO}}_1.$\hfill$\blacksquare$
  \end{enumerate}
  }
\end{definition}

\noindent
The notion of equivalence adopted here is different from the one in \cite{7933195}. Here the equivalence is considered in the semantic space and not just based on syntactic names of input-outputs.

\begin{example}{
 Consider two services ${\mathcal{S}}_1$ and ${\mathcal{S}}_3$ as shown in Table \ref{tab:ourData}. From the definitions of ${\mathcal{S}}_1$ and ${\mathcal{S}}_3$, ${\mathcal{S}}_1 \simeq {\mathcal{S}}_3$, since according to ontology, $f(BWImage) = f(binImage) = binaryImage$. Moreover, BWImage in jpeg format $\leftrightarrow$ BinImage in jpeg format and BWImage size must be less than $50KB \leftrightarrow$  BinImage size must be less than $50KB$.
 \hfill$\blacksquare$}
\end{example}

\noindent
Two services ${\mathcal{S}}_1$ and ${\mathcal{S}}_2$ satisfying the first three conditions of Definition \ref{def:Equi} are called {\em input equivalent}, whereas, two services satisfying the last three conditions of the definition are called {\em output equivalent}. 
 
It may be noted, the binary relation $``\simeq"$ defined over a set of services is an equivalence relation. The first level of abstraction is done based on the equivalence relationship between services. Since $``\simeq"$ is an equivalence relation, the service repository is partitioned based on this relation. Each equivalence class is then abstracted by a new representative service. The number of services required to serve a query reduces by this abstraction. Since the service repository is partitioned into a set of equivalence classes and the equivalence classes are mutually exclusive and collectively exhaustive, the number of equivalence classes in the service repository is less than or equal to the number of services in the repository.

% \begin{lemma}\label{lemma:firstLevelReduction}
%  The number of abstract services after the first level of abstraction is less than or equal to the original number of services.
%  \hfill$\blacksquare$
% \end{lemma}
% \noindent
% The proof directly follows from Lemma \ref{lemma:equivalence}.  

Consider ${\mathcal{S}}' \subseteq S$ be a set of equivalent services. Using the above method, ${\mathcal{S}}'$ is abstracted by a new service, say ${\mathcal{S}}1_{i}$. To assign the inputs, outputs of the abstract service, we choose any member service (${\mathcal{S}}_{i} \in {\mathcal{S}}'$) from its corresponding equivalence class (since each member service is functionally equivalent) and assign the following: 
{ ${\mathcal{S}}1^{(ip)}_{i} = \{f(j) : \forall j \in {\mathcal{S}}^{ip}_{i}\}; ~~{\mathcal{S}}1^{(op)}_{i} = \{f(o) : \forall o \in {\mathcal{S}}^{op}_{i}\};$ ${\mathcal{PR}}1_{i} = {\mathcal{PR}}_{i}; ~~{\mathcal{PO}}1_{i} = {\mathcal{PO}}_{i}$}.
We now illustrate this abstraction on our working example.

\begin{example}{{
 Considering our running example (shown in Table \ref{tab:ourData}), the dependency graph constructed in response to ${\mathcal{Q}}$ (shown in Example \ref{example:query}) with the first level abstract services is shown in Figure \ref{fig:exampleLevel1}. 
 \begin{figure}
\centering
\includegraphics[width=\linewidth]{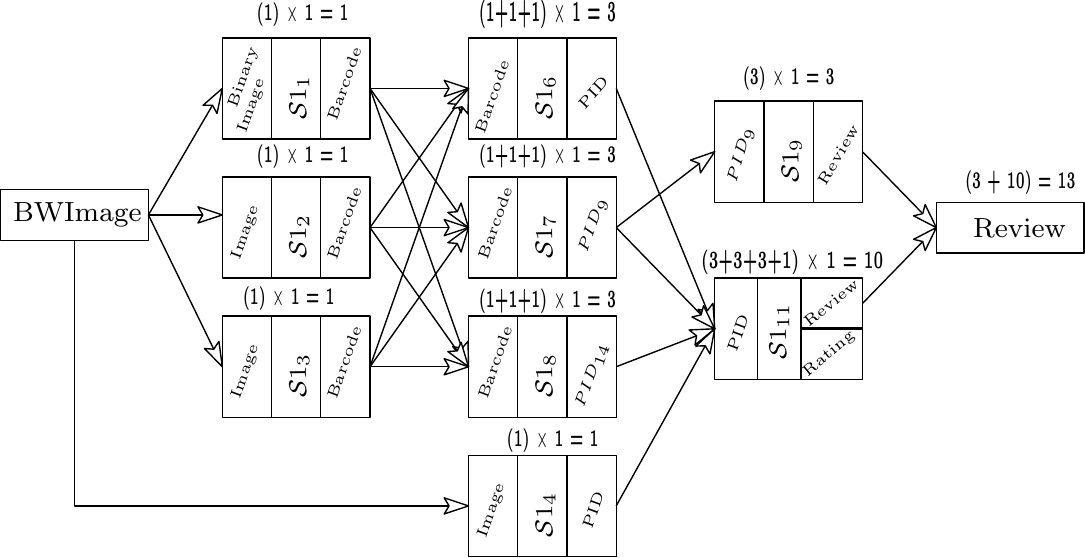}
\caption{{Dependency Graph using first level abstract services to respond ${\mathcal{Q}}$}}
\label{fig:exampleLevel1}
\end{figure}
After this abstraction, the dependency graph now consists of 9 services and the total number of solutions (in terms of functional dependencies) is now reduced from 173 to 13. The calculation is shown in the figure.
% {
% ($|((\{{\mathcal{S}}1_1, {\mathcal{S}}1_2, {\mathcal{S}}1_3\} \times \{{\mathcal{S}}1_{6}, {\mathcal{S}}1_{7}, {\mathcal{S}}1_{8}\}) \cup \{{\mathcal{S}}1_{4}\})  \times \{{\mathcal{S}}1_{9}, {\mathcal{S}}1_{11}\}| = $
% $(3 \times 3 + 1) \times 2 \times 1 = 20$.}).
 \hfill$\blacksquare$}}
\end{example}

Once an abstract service ${\mathcal{S}}1_{i}$ is constructed for each equivalence class, it is also important to assign the QoS values to ${\mathcal{S}}1_{i}$. To do so, we choose a representative service from each corresponding equivalence class and assign the QoS values of the representative service to it. Consider ${\mathcal{S}}1_{i}$ is an abstract service corresponding to the equivalence class ${\mathcal{S}}' \subseteq S$. Below, we discuss the method of choosing the representative service.

\noindent
{\bf Step 1:} For each service ${\mathcal{S}}_{i}$ in an equivalence class, we first normalize the values of the QoS parameters between 0 and 1. To normalize, at first, the maximum (${\mathcal{P}}^{Max}_l$) and minimum (${\mathcal{P}}^{Min}_l$) values of each QoS parameter ${\mathcal{P}}_l$ across the equivalence class are computed. The normalized value of the parameter is computed as: 
 
\noindent
{
$NV({\mathcal{P}}^{i}_l) = 
\begin{cases}
\frac{{\mathcal{P}}^{i}_l - {\mathcal{P}}^{Min}_l}{{\mathcal{P}}^{Max}_l - {\mathcal{P}}^{Min}_l}& \text{, if } {\mathcal{P}}_l \text{is a positive QoS} \\
\frac{{\mathcal{P}}^{Max}_l - {\mathcal{P}}^{i}_l}{{\mathcal{P}}^{Max}_l - {\mathcal{P}}^{Min}_l}& \text{, if } {\mathcal{P}}_l \text{is a negative QoS}
\end{cases}$  }
      
\noindent
{\bf Step 2:} Once the QoS parameters are normalized, we compute the weighted sum of the normalized QoS parameters:

{$WS ({\mathcal{S}}_{i}) = \sum \limits_{l=1}^m (NV({\mathcal{P}}^{i}_l) * w_l); ~~~ w_l \in [0, 1]$}

\noindent
We assume that the weights of the QoS parameters can be provided externally by the user. If no weights are specified, all QoS parameters are provided equal weights, i.e., $\forall l \in \{1, 2. \ldots, m\}, w_l = \frac{1}{m}$.

\noindent
{\bf Step 3:} The service with maximum weighted sum is chosen as the representative service for each equivalence class. 

\subsection{\bf \em{Composition with First Level Abstract Services}}
\noindent 
Any service composition algorithm can be applied over the abstract service space to get a solution to a query. The major advantage of this abstraction method is that for a single QoS parameter, it preserves the solution quality. In other words, if an optimal algorithm is used to find a solution, it is guaranteed to find an optimal solution from the abstract service space as well. Similarly, if a heuristic algorithm is used to find a solution from the abstract space, it is guaranteed that the solution quality does not degrade in comparison to the solution that can be generated from the original / un-abstracted service space by the same algorithm. However, for multiple QoS parameters, we cannot ensure this guarantee. In some scenarios, it may not be possible to generate any solution from the abstract service space that satisfies all the QoS constraints, though one exists in the original service space. In such scenarios, refinement is required. In the next subsection, we discuss the refinement techniques in detail.

\subsection{\bf \em{First Level Refinement}}
\noindent 
In this paper, we consider two refinement techniques: partial refinement and complete refinement.

\subsubsection*{{\bf Partial Refinement}}
\noindent
If the composition algorithm returns a solution in terms of the abstract services which does not satisfy all the QoS constraints, we first apply the partial refinement on the solution. It may be noted, one solution in terms of abstract services consists of multiple solutions in terms of original services. The main idea of the partial refinement technique is to change the representative service corresponding to each abstract service belonging to the solution and thus to change the QoS values of the solution. As discussed earlier, the representative service is chosen on the basis of the weighted sum method. Here, we change the weights of the QoS parameters on the basis of the constraint values to choose a different service. We first consider a solution $s$ consisting of the set of abstract services ${\mathcal{S}}1_s$. We now discuss how we modify the weights of the QoS parameters.

\noindent
{\bf{Step 1:}} For each abstract service ${\mathcal{S}}1_{i} \in {\mathcal{S}}1_s$, we find the maximum $Max({\mathcal{S}}1_{i}, {\mathcal{P}}_l)$ and the minimum $Min({\mathcal{S}}1_{i}, {\mathcal{P}}_l)$ values for each QoS parameter ${\mathcal{P}}_l$ from its corresponding equivalence class.

\noindent
{\bf{Step 2:}} For the solution $s$, we first compute the maximum aggregated value $Aggr_{max}(s, {\mathcal{P}}_l)$ for each QoS parameter ${\mathcal{P}}_l$, considering $Max({\mathcal{S}}1_{i}, {\mathcal{P}}_l)$ as the value of ${\mathcal{P}}_l$ of ${\mathcal{S}}1_{i}$, for all ${\mathcal{S}}1_{i} \in {\mathcal{S}}1_s$. Similarly, we compute the minimum aggregated value $Aggr_{min}(s, {\mathcal{P}}_l)$, for each QoS parameter ${\mathcal{P}}_l$, considering $Min({\mathcal{S}}1_{i}, {\mathcal{P}}_l)$ as the value of ${\mathcal{P}}_l$ of ${\mathcal{S}}1_{i}$, for all ${\mathcal{S}}1_{i} \in {\mathcal{S}}1_s$. 

\noindent
{\bf{Step 3:}} Among $Aggr_{max}(s, {\mathcal{P}}_l)$ and $Aggr_{min}(s, {\mathcal{P}}_l)$, we consider one as the best value and the other as the worst value of ${\mathcal{P}}_l$ for the solution. If ${\mathcal{P}}_l$ is a positive (negative) QoS parameter, $Aggr_{max}(s, {\mathcal{P}}_l)$ is considered to be the best (worst) value and $Aggr_{min}(s, {\mathcal{P}}_l)$ is considered to be the worst (best) value of ${\mathcal{P}}_l$. 

\noindent
{\bf{Step 4:}} If the best value of any QoS parameter ${\mathcal{P}}_l$ of the solution does not satisfy the constraint, we do not apply the partial refinement, since, no modification of the representative service corresponding to any abstract service can make the constraint satisfiable. 

\noindent
{\bf{Step 5:}} If the worst value of any QoS parameter ${\mathcal{P}}_l$ of the solution is better than the constraint value, we ignore the QoS parameter ${\mathcal{P}}_l$ from the weight computation. This is mainly because, no matter what representative service is chosen for an abstract service belonging to the solution, the constraint value is always satisfied. 

\noindent
{\bf{Step 6:}} If the constraint value of ${\mathcal{P}}_l$ (say $\delta({\mathcal{P}}_l)$) lies between the best and the worst aggregated value, we compute the normalized value of the constraint as:
 
 \noindent
%  $$
 \[NV({\mathcal{P}}_l) =
 \begin{cases}
    0 & \text{ if Step 5 is satisfied  }\\
    \frac{\delta({\mathcal{P}}_l) - Aggr_{min}(s, {\mathcal{P}}_l)}{Aggr_{max}(s, {\mathcal{P}}_l) - Aggr_{min}(s, {\mathcal{P}}_l)} & \text{if } {\mathcal{P}}_l \text{ is a positive QoS}\\
    \frac{Aggr_{max}(s, {\mathcal{P}}_l) - \delta({\mathcal{P}}_l)}{Aggr_{max}(s, {\mathcal{P}}_l) - Aggr_{min}(s, {\mathcal{P}}_l)} & \text{if } {\mathcal{P}}_l \text{ is a negative QoS}
  \end{cases}
 \]
 \\~
 
 \noindent
 It may be noted, when we assign the QoS values to an abstract service during the preprocessing phase, we do not have any information regarding the QoS requirements of a query, which can only be known when a query comes to the system during the execution phase. Moreover, the QoS constraints vary across queries. Therefore, in the partial refinement phase, our intuition is to set the weights of the QoS parameters according to the QoS requirements by the query. Since the QoS parameters are disparate in nature, in order to compare them, we first normalize a QoS constraint value with respect to the maximum and minimum aggregated values of the parameter that can be provided by the solution. It may be noted, a solution can provide maximum normalized QoS value as 1 for any QoS parameter. However, a QoS parameter ${\mathcal{P}}_l$ requires the normalized value as $NV({\mathcal{P}}_l)$ to satisfy its corresponding constraint. Therefore, we set the weight of the QoS parameters proportionate to their normalized constraint values.

\noindent
{\bf{Step 7:}} The weight for each QoS parameter is finally computed as, 
 \[WT({\mathcal{P}}_l) = \frac{NV({\mathcal{P}}_l)}{\sum \limits_{k=1}^m NV({\mathcal{P}}_l)}\]

\noindent
Once we compute the weights of the QoS parameters, we reiterate the algorithm for choosing the representative service as discussed above and assign the QoS values of the representative service to its corresponding abstract service. Finally, we compute the aggregated QoS values of the solution with the modified abstract services. If the solution now satisfies all the QoS constraints, we return the solution. Otherwise, we apply complete refinement as discussed below.

\begin{example}{
 Consider a solution to a query as shown in Figure \ref{fig:refinement}, which is a sequential combination of two abstract services ${\mathcal{S}}1_1$ and ${\mathcal{S}}1_2$, each corresponding to 3 un-abstract services. Each service has 2 QoS parameters: response time and reliability (as shown by the tuple in the figure). The calculated weighted sum for each service is also shown in figure. For each abstract service, the corresponding representative service is shown by gray color. The QoS values of the solution is $(80, 0.76)$. Now consider the following two cases: 
 
 \begin{figure}[!htb]
  \centering
    \includegraphics[]{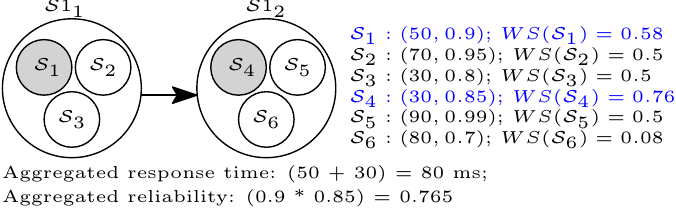}
  \caption{A solution using abstract services}
  \label{fig:refinement}
  \end{figure}
 
 \noindent
 {\bf Case 1:} The query has the following constraints: $(200, 0.8)$. Clearly, the solution violates the reliability constraint. Therefore, we apply partial refinement. The maximum and minimum values of the response time corresponding to ${\mathcal{S}}1_1$ are 70 and 30ms respectively and the same for ${\mathcal{S}}1_2$ are 90 and 30ms respectively. Therefore, the maximum aggregated response time of the solution is 160ms (90 + 70) and the minimum aggregated value for the solution is 60ms (30 + 30). Since the response time constraint is 200ms, we can ignore response time for the weight calculation. This is because whatever representative service is chosen for an abstract service, the response time constraint remains satisfied.
 Similarly, the maximum and minimum values of the reliability corresponding to ${\mathcal{S}}1_1$ are 0.95 and 0.8 respectively and the same for ${\mathcal{S}}1_2$ are 0.99 and 0.7 respectively. Therefore, the maximum aggregated reliability of the solution is 0.94 (0.95 * 0.99) and the minimum aggregated value for the solution is 0.56 (0.8 * 0.7). The constraint value lies between 0.94 and 0.56. Therefore, we first calculate the normalized constraint value for reliability $\frac{0.8 - 0.56}{0.94 - 0.56} = 0.63$. Following the above algorithm, the weights for response time and reliability are 0 and 1 ($0.63 / 0.63$). Using the calculated weights, we reiterate the algorithm for choosing the representative service. ${\mathcal{S}}_2$ is now the representative service for ${\mathcal{S}}1_1$ and ${\mathcal{S}}_5$ is the representative service for ${\mathcal{S}}1_2$, since they have now maximum weighted sum compared to the other members belonging to their corresponding equivalence classes. Hence, the solution now has the QoS values $(160, 0.94)$, which satisfies all the constraints.\\
 
 \noindent
 {\bf Case 2:} The query has the following constraints: $(50, 0.8)$. In this case, the response time constraint can never be satisfied using these two abstract services, since the best response time value for the solution is 60ms. Therefore, complete refinement is required.
 \hfill$\blacksquare$}
\end{example}

\subsubsection*{{\bf Complete Refinement}}
\noindent
If the composition algorithm does not return any solution or the partial refinement technique fails to generate a solution, we apply the complete refinement technique. The main idea of the complete refinement is to revert the abstract services to the original services and apply the composition algorithm on the original service space once more to obtain the solution.

\subsection{\bf \em{Second level Abstraction: Based on Dominance Relation}}
\noindent
This abstraction is done based on the notion of {\em dominance}. In this abstraction, we cluster the set of abstract services formed after the step above into further groups, where one service dominates the rest of the services. 

\begin{definition}{\bf[Dominance]:}\label{def:dom}{\em~
 A service ${\mathcal{S}}1_1$ dominates another service ${\mathcal{S}}1_2$, expressed as ${\mathcal{S}}1_1 \succ {\mathcal{S}}1_2$, if the following conditions hold:
 \begin{itemize}
   \item $\forall i \in {\mathcal{S}}1^{ip}_1, \exists i' \in {\mathcal{S}}1^{ip}_2$, such that $i \sqsupseteq i'$. 
   \item $\forall o' \in {\mathcal{S}}1^{op}_2, \exists o \in {\mathcal{S}}1^{op}_1$, such that $o \sqsubseteq o'$. 
   \item ${\mathcal{PR}}1_2 \rightarrow {\mathcal{PR}}1_1$. ${\mathcal{PO}}1_1 \rightarrow {\mathcal{PO}}1_2$.
   \item ${\mathcal{S}}1_1$ and ${\mathcal{S}}1_2$ are not output equivalent.
  \hfill$\blacksquare$
 \end{itemize}}
\end{definition}

\noindent
The notion of dominance essentially indicates that the functionality of ${\mathcal{S}}1_2$ is covered by the functionality of ${\mathcal{S}}1_1$. In other words, ${\mathcal{S}}1_1$ can be activated, if ${\mathcal{S}}1_2$ gets activated and ${\mathcal{S}}1_1$ provides more functionality than ${\mathcal{S}}1_2$ (either provides more outputs or provides more specific outputs, i.e. at least one output of ${\mathcal{S}}1_1$ is a sub concept of ${\mathcal{S}}1_2$ or the postcondition of ${\mathcal{S}}1_1$ implies the postcondition of ${\mathcal{S}}1_2$). The last condition is used for the next abstraction level.

\begin{example}{\em~
 Consider two services ${\mathcal{S}}1_9$ and ${\mathcal{S}}1_{11}$, as shown in Figure \ref{fig:exampleLevel1}. From their definitions, ${\mathcal{S}}1_{11} \succ {\mathcal{S}}1_9$, since according to ontology, $PID_9 \sqsubset PID$, ${\mathcal{PR}}1_{9} \rightarrow {\mathcal{PR}}1_{11}$ and \{Review\} $\subset$ \{Review, Rating\}. It may be noted, ${\mathcal{S}}1_{11}$ must be activated, if ${\mathcal{S}}1_9$ gets activated and ${\mathcal{S}}1_{11}$ provides more outputs than ${\mathcal{S}}1_9$.
 \hfill$\blacksquare$}
\end{example}

\noindent
Consider ${\mathcal{S}}1$ is the set of first level abstract services. If a service ${\mathcal{S}}1_i \in {\mathcal{S}}1$ is not dominated by any other service in ${\mathcal{S}}1$ (i.e., ${\mathcal{S}}1_i$ is non-dominated), ${\mathcal{S}}1_i$ forms a group consisting of all the first level abstract services that are dominated by ${\mathcal{S}}1_i$. The group is finally abstracted by a new service, say ${\mathcal{S}}2_i$. The input, output parameters, preconditions and postconditions of ${\mathcal{S}}2_i$ are same as ${\mathcal{S}}1_i$. The QoS values of ${\mathcal{S}}1_i$ are assigned to ${\mathcal{S}}2_i$. It may be noted, though, every dominated service ${\mathcal{S}}2_j$ belongs to at least one group, however, no such service forms its separate group containing all the first level abstract services that are dominated by ${\mathcal{S}}2_j$. The set of services corresponding to ${\mathcal{S}}2_i$ and ${\mathcal{S}}2_j$ $(i \neq j)$ are not always mutually exclusive, since one service can be dominated by multiple services. However, the number of services after the second level of abstraction still reduces. This is because the maximum number of non-dominated services is at most equal to the number of services after the first level.

\begin{example}{ \em
 Considering our running example (shown in Table \ref{tab:ourData}), the dependency graph constructed in response to ${\mathcal{Q}}$ (shown in Example \ref{example:query}) with the second level abstract services is shown in Figure \ref{fig:exampleLevel3}.
 \begin{figure}[!htb]
 \centering
 \includegraphics[width=\linewidth]{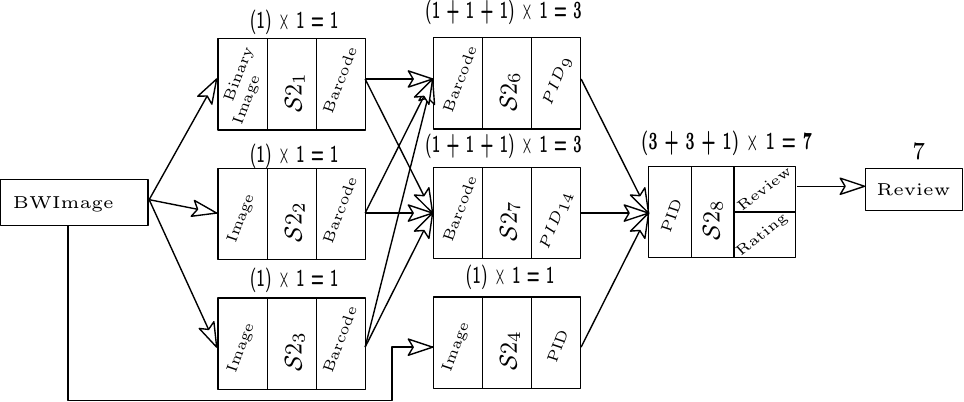}
 \caption{{Dependency Graph using second level abstract services to respond ${\mathcal{Q}}$}}
 \label{fig:exampleLevel3}
 \end{figure}
The dependency graph now consists of 7 services and the total number of solutions (in terms of functional dependencies) is now reduced from 13 to 7. The calculation is shown in the figure.
% {
% ($|((\{{\mathcal{S}}2_1, {\mathcal{S}}2_2, {\mathcal{S}}2_3\} \times \{{\mathcal{S}}2_{6}, {\mathcal{S}}2_{7}\}) \cup \{{\mathcal{S}}2_{4}\})  \times \{{\mathcal{S}}2_{8}\}| = $
% $(3 \times 2 + 1) \times 1= 7$.}).
 \hfill$\blacksquare$}
\end{example}

\noindent
As earlier, any service composition algorithm can be applied over the abstract service space to get a solution to a query. However, unlike the first level abstraction, the second level abstraction is not quality preserving. Therefore, refinement may be required even for a single QoS parameter.

% The refinement technique for this abstraction level is exactly same as the first level abstraction technique.

\subsection{\bf \em{Second Level Refinement}}
\noindent 
In this level as well, we have two different refinement techniques: partial refinement and complete refinement. The partial refinement here is similar to the previous level. The only difference is that for a second level abstract service ${\mathcal{S}}2_i$, all its corresponding first level abstract services do not participate in selecting the representative service for ${\mathcal{S}}2_i$.
This is because of the following reason: consider ${\mathcal{S}}1_i$ is the non dominated first level service belonging to the group corresponding to ${\mathcal{S}}2_i$. First of all, it is not necessary, when ${\mathcal{S}}1_i$ gets activated all other services in the group are also activated with the available inputs in the system. Furthermore, ${\mathcal{S}}1_i$ provides more functionality than any other first level abstract services in that group, which may require serving the query. In this case, no other service can participate in serving the query. Therefore, we have to first select the set of services corresponding to a second level abstract service that can serve the query.

Consider a second level abstract service ${\mathcal{S}}2_i$ belonging to the solution corresponding to the set of first level abstract services: ${\mathcal{S}}1' \in {\mathcal{S}}1$, where ${\mathcal{S}}1$ is the set of first level abstract services. 
A first level abstract service ${\mathcal{S}}1_{i} \in {\mathcal{S}}1'$ participates in the representative service selection method, if the following conditions are satisfied: 
(1) ${\mathcal{S}}1_{i}$ can be activated from the available input set, when ${\mathcal{S}}2_i$ is activated.
(2) ${\mathcal{S}}1_{i}$ produces all the parameters (in the form of identical / sub concept) that are required to be produced by ${\mathcal{S}}2_i$ only, i.e., no other service in the solution produces them. More formally,
$\forall io_j \in IO, \exists io_l \in {\mathcal{S}}1^{op}_{i}, io_k \sqsupseteq io_l$, where $IO$ is the set of parameters that are required to be produced by ${\mathcal{S}}2_i$ only. (3)The post conditions of ${\mathcal{S}}1_{i}$ must satisfy the set of conditions that are required to be ensured by ${\mathcal{S}}2_i$. More formally, ${\mathcal{PO}}1_i \rightarrow {\mathcal{PC}}$, where, ${\mathcal{PC}}$ is the set of conditions that are required to be ensured by ${\mathcal{S}}2_i$.

% \begin{itemize}
%  \item ${\mathcal{S}}1_{i}$ can be activated from the available input set, when ${\mathcal{S}}2_i$ is activated.
%  \item ${\mathcal{S}}1_{i}$ produces all the parameters (in the form of identical / sub concept) that are required to be produced by ${\mathcal{S}}2_i$ only, i.e., no other service in the solution produces them. More formally,
%  $\forall io_j \in IO, \exists io_l \in {\mathcal{S}}1^{op}_{i}, io_k \sqsupseteq io_l$, where $IO$ be the set of parameters that are required to be produced by ${\mathcal{S}}2_i$ only.
% %  \item Consider $IO$ be the set of input-output parameters that are required to be produced by ${\mathcal{S}}2_i$ only, i.e., no other service in the solution produces the set $IO$. ${\mathcal{S}}1_{i}$ produces all the parameters in $IO$ (either in the form of identical concept or in the form of sub concept). More formally,
% %  $\forall io_j \in IO, \exists io_l \in {\mathcal{S}}1^{op}_{i}, io_k \sqsupseteq io_l$.
%  \item The post conditions of ${\mathcal{S}}1_{i}$ must satisfy the set of conditions that are required to be ensured by ${\mathcal{S}}2_i$. More formally, ${\mathcal{PO}}1_i \rightarrow {\mathcal{PC}}$, where, ${\mathcal{PC}}$ is the set of conditions that are required to be ensured by ${\mathcal{S}}2_i$.
% %  \item Consider ${\mathcal{PC}}$ be the set of conditions that are required to be ensured by ${\mathcal{S}}2_i$. The post conditions of ${\mathcal{S}}1_{i}$ must satisfies ${\mathcal{PC}}$. More formally,
% %  ${\mathcal{PO}}1_i \rightarrow {\mathcal{PC}}$.
% \end{itemize}

\noindent
Once the representative services are chosen, the input-output parameters, pre and postconditions and the QoS values are changed according to the chosen representative services. If a valid solution is obtained, we return the solution. Otherwise, we apply the complete refinement technique. 

\subsection{{\bf \em Abstraction Using Input Implication Output Equivalence}}
\noindent
This is the third and the final level of abstraction. The abstraction is done based on the notion of input equivalence and output implication. Here, we focus on the set of services that produce equivalent set of outputs and ensure the same postconditions, however, one service dominates the other service in terms of inputs and preconditions.

\begin{definition}{\bf[Input Implication Output Equivalence Relation (IIOE)]}:{\em~
 IIOE is a relation between two services ${\mathcal{S}}2_1$ and ${\mathcal{S}}2_2$ (expressed as ${\mathcal{S}}2_1 \xrightarrow[\text{IIOE}]{} {\mathcal{S}}2_2$), which ensures ${\mathcal{S}}2_1$ is output equivalent to ${\mathcal{S}}2_2$ and ${\mathcal{S}}2_1$ implies ${\mathcal{S}}2_2$ in terms of input parameters satisfying the following conditions:
 \begin{itemize}
%   \item $\forall i \in {\mathcal{S}}1^{ip}_2, \exists i' \in {\mathcal{S}}1^{ip}_1$, such that $f(i) = f(i')$.
  \item $\forall i \in {\mathcal{S}}2^{ip}_2, \exists i' \in {\mathcal{S}}2^{ip}_1$, such that $i \sqsupseteq i'$.
  \item ${\mathcal{PR}}2_1 \rightarrow {\mathcal{PR}}2_2$.
 \hfill$\blacksquare$.  
 \end{itemize}}
\end{definition}

\noindent
The above definition means, though both the services ${\mathcal{S}}2_1$ and ${\mathcal{S}}2_2$ provide same output and ensure same functionality, however, ${\mathcal{S}}2_1$ is more generic than ${\mathcal{S}}2_2$ in a sense that ${\mathcal{S}}2_1$ can be activated with less number of inputs or with weaker preconditions \cite{DBLP:conf/icws/JuSR12}.

\begin{example}{\em~
 Consider two services ${\mathcal{S}}2_1$ and ${\mathcal{S}}2_3$, as shown in Figure \ref{fig:exampleLevel3}. 
 From their definitions, ${\mathcal{S}}2_1 \xrightarrow[\text{IIOE}]{} {\mathcal{S}}2_3$, since both the services are output equivalent and
 according to the ontology, $BinaryImage \sqsubset Image$ and ${\mathcal{PR}}2_1 \rightarrow {\mathcal{PR}}2_2$.
 \hfill$\blacksquare$}
\end{example}

\noindent
We first construct an IIOE graph based on the IIOE relation between the services.

\begin{definition}{\bf[IIOE Graph]}:{\em~
 An IIOE graph $G_{IIOE} = (V, E)$ is a directed acyclic graph, where $V$ is the set of nodes and $E = (u_i, u_j); u_i, u_j \in V$ is the set of edges. Each node $u_i \in V$ contains a first-level abstract service ${\mathcal{S}}2_i$. An edge from $u_i$ to $u_j$ exists, if ${\mathcal{S}}2_i \xrightarrow[\text{IIOE}]{} {\mathcal{S}}2_j$.
 \hfill$\blacksquare$}
\end{definition}

\noindent
We first create all possible nodes and edges in the IIOE graph. Finally, we remove all the transitive edges to make the composition algorithm efficient, as discussed later. If ${\mathcal{S}}2_i \rightarrow {\mathcal{S}}2_j$ and ${\mathcal{S}}2_j \rightarrow {\mathcal{S}}2_k$, we remove the edge from $u_i$ to $u_k$. 

Since, each node $u_i \in V$ contains a second-level abstract service ${\mathcal{S}}2_i$, the number of nodes in $G_{IIOE}$ is equal to the number of second-level abstract services. Now consider a node $u_i \in V$ in $G_{IIOE}$ containing an abstract service ${\mathcal{S}}2_i$. It may be noted, if ${\mathcal{S}}2_i$ is activated in the system, all the services corresponding to the {\em tree} rooted at $u_i$ are also activated at the same time, and provide same functionality as ${\mathcal{S}}2_i$. Therefore, the entire tree is abstracted by a new service, say ${\mathcal{S}}3_i$. In other words, each final level abstract service ${\mathcal{S}}3_i$ corresponds to a tree rooted at the node $u_i$ in the implication graph. Therefore, the number of final level abstract services is equal to the number of second level abstract services. The input, output parameters, pre and postconditions of ${\mathcal{S}}3_i$ are exactly same as ${\mathcal{S}}2_i$ and the QoS assignment technique for this abstraction level is same as the first level abstraction, i.e., for each final level abstract service ${\mathcal{S}}3_i$, we choose a representative service from its corresponding set of second level services and assign its QoS values to ${\mathcal{S}}3_i$.

\begin{example}{\em~
 Consider three services ${\mathcal{S}}2_1$, ${\mathcal{S}}2_2$ and ${\mathcal{S}}2_3$, as shown in Figure \ref{fig:exampleLevel3}. From their definitions, ${\mathcal{S}}2_1 \rightarrow {\mathcal{S}}2_2 \rightarrow {\mathcal{S}}2_3$, since according to ontology, $BinaryImage \sqsubset Image$ and ${\mathcal{PR}}2_1 \rightarrow {\mathcal{PR}}2_3$, ${\mathcal{PR}}2_2 \rightarrow {\mathcal{PR}}2_3$. $G_{IIOE}$ constructed from the services is shown in Figure \ref{fig:exampleLevel4}. Each third level abstract service is shown with a different color in the figure. Considering our running example (shown in Table \ref{tab:ourData}), the dependency graph constructed in response to ${\mathcal{Q}}$ (shown in Example \ref{example:query}) with the final level abstract services is shown in Figure \ref{fig:exampleLevel4}.   
 \begin{figure}[!htb]
 \centering
 \includegraphics[width=\linewidth]{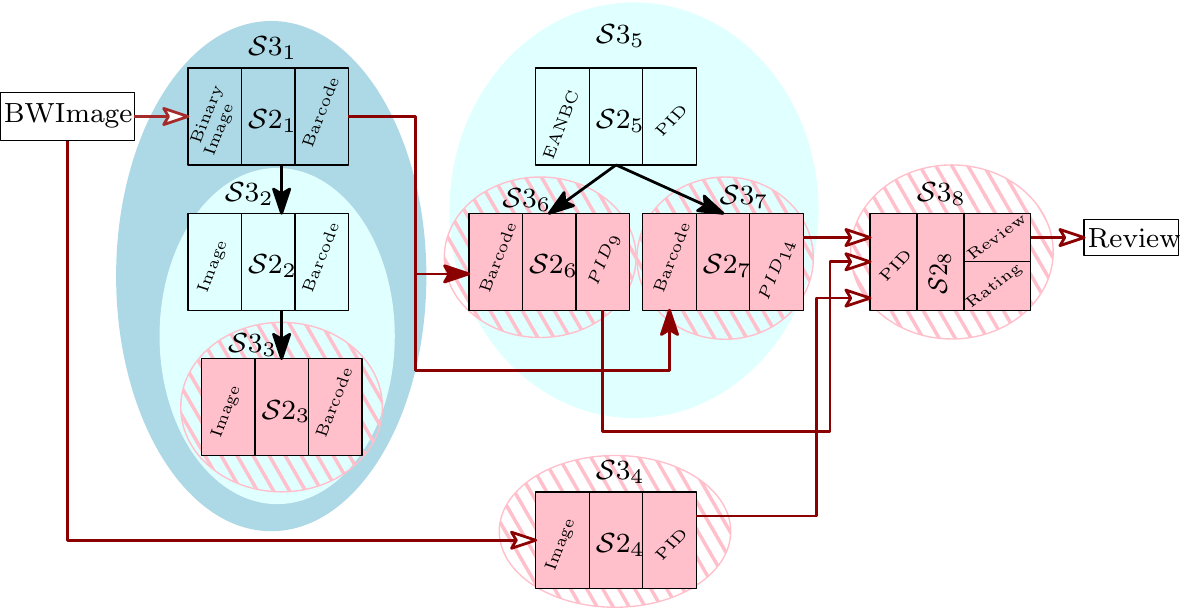}
 \caption{{Dependency Graph using final level abstract services to respond to ${\mathcal{Q}}$}}
 \label{fig:exampleLevel4}
 \end{figure}
After the final level of abstraction, the dependency graph now consists of 5 services and the total number of solutions (in terms of functional dependencies) is now reduced from 7 to 3
{
($|((\{{\mathcal{S}}3_1\} \times \{{\mathcal{S}}3_{6}, {\mathcal{S}}3_{7}\}) \cup \{{\mathcal{S}}3_{4}\})  \times \{{\mathcal{S}}3_{8}\}| = $
$(1 \times 2 + 1) \times 1 \times 1 = 3$.}).
 \hfill$\blacksquare$}
\end{example}

\noindent
Though the number of services in this step is exactly same as in the previous step, the number of nodes in the dependency graph reduces from the previous level of abstraction. During composition, if one final-level abstract service ${\mathcal{S}}3_i$ is chosen in the dependency graph, no other final level service is chosen for which its corresponding second level service belongs to the tree rooted at the node containing ${\mathcal{S}}2_i$, where ${\mathcal{S}}3_i$ corresponds to the tree rooted at ${\mathcal{S}}2_i$. Therefore, the composition algorithm becomes efficient, if it starts traversing the graph top down (i.e., starting from the root node of any connected component of the IIOE graph) during dependency graph construction.

\vspace{0.1cm}
\begin{example}{\em~
 As shown in Figure \ref{fig:exampleLevel4}, if ${\mathcal{S}}3_1$ is chosen to serve a query, ${\mathcal{S}}3_2$ and ${\mathcal{S}}3_3$ are not chosen. Similarly, if ${\mathcal{S}}3_5$ is chosen, ${\mathcal{S}}3_6$ and ${\mathcal{S}}3_7$ are not chosen. 
 \hfill $\blacksquare$}  
\end{example}

\noindent
The final level refinement techniques are exactly same as the first level refinement techniques.

% \input{secondLevel}
% \input{thirdLevel}
% \begin{figure*}
% \centering
% (a)\includegraphics[width=0.3\linewidth]{fig/noOfServiceInDG.eps}
% (b)\includegraphics[width=0.3\linewidth]{fig/RT.eps}
% (c)\includegraphics[width=0.3\linewidth]{fig/TR.eps}
% \caption{\scriptsize{(a) No. of services in dependency graph (b) Response time (c) Throughput}}
% \label{fig:analysisPerformance}
% \end{figure*}

\subsection{\bf \em{Solution Reconstruction and Analysis}}
\noindent
Once a solution is generated in terms of the abstract services, we revert the solution from the abstract space to the original service space. The abstract service in each level is replaced by its previous level service until we reach the original service space. Our framework provides the following guarantees.

% \begin{figure}[!htb]
% \centering
% \includegraphics[width=\linewidth]{fig/abs.pdf}
% \caption{Flow of composition}
% \label{fig:flow}
% \end{figure}

%  \begin{algorithm}
%    \caption{SolutionConstruction}
%    \begin{algorithmic}[1]
%     \State {Input: Solution in terms of abstract services $s$}
%     \State {Output: Solution in terms of original services}
%     \For {each service ${\mathcal{S}}2_i \in s$} 
%        \State ${\mathcal{S}}2_i$ is replaced by its corresponding first level abstract \indent service ${\mathcal{S}}1_j$;
%        \State ${\mathcal{S}}1_j$ is replaced by its corresponding representative \indent service ${\mathcal{S}}_k$;
%     \EndFor
%    \end{algorithmic}
%   \label{algo:revert}
% \end{algorithm}
% %====================================

% \subsection{\bf \em{Analysis of our methodology}}
% \noindent 

\begin{lemma}{{\textbf{\em{Soundness}}}: }\label{soundAbstraction}
 A QoS constraint satisfying solution containing abstract services produced by any composition algorithm is a valid solution satisfying all QoS constraints on the original service space. 
 \hfill$\blacksquare$
\end{lemma}

\begin{proof}
 Once a QoS constraint satisfying solution is generated in terms of the abstract services, the solution is reverted from the abstract space to the original service space. It may be noted, each $i^{th}$ ($i = 1, 2, 3$) level abstract service of the solution, containing the $i^{th}$ level abstract services, is replaced by its corresponding previous level service having the same QoS values. Therefore, the overall QoS values of the solution remains unchanged on the original service space. Hence, we always get a valid solution satisfying all QoS constraints on the original service space.
\end{proof}

\begin{lemma}
 A solution (in terms of functional dependencies) to a query can always be constructed by composing the abstract services, if and only if there exists a solution to the query in the original service space. 
 \hfill$\blacksquare$
\end{lemma}

\begin{proof}
 We first consider the following cases:
 \begin{itemize}
  \item If any service ${\mathcal{S}}_i$ is eventually activated from the query inputs and input specification, the abstract service ${\mathcal{S}}1_j$ containing ${\mathcal{S}}_i$ must be activated, since both of them are functionally equivalent and will produce the set of outputs and identical postconditions, that further producing the query outputs eventually with desired output requirements, if there exists a solution to the query.
  \item If any service ${\mathcal{S}}1_i$ is eventually activated from the query inputs and input specification, the abstract service ${\mathcal{S}}2_j$ containing ${\mathcal{S}}1_i$ must be activated, since the representative service of ${\mathcal{S}}2_j$ either  ${\mathcal{S}}1_i$ itself or dominates ${\mathcal{S}}1_i$. Therefore, ${\mathcal{S}}2_j$ will produce the set of outputs containing the output set of ${\mathcal{S}}1_i$ having stronger postconditions, that further producing the query outputs eventually with desired output requirements, if there exists a solution to the query.
  \item For each service ${\mathcal{S}}2_i$, there exist a third level abstract service ${\mathcal{S}}3_i$ having identical set of inputs, outputs, preconditions and postconditions. Therefore, the same argument holds for this case as well. 
 \end{itemize}
 Therefore, if there exists a solution to a query in the original service space, we can always construct a solution in terms of the functional dependencies by composing any level of abstract services. 
\end{proof}

Our entire framework provides the following guarantee.

\begin{lemma}\label{complete}
 Our framework is always able to find a valid solution, if one exists. 
 \hfill$\blacksquare$
\end{lemma}

\noindent
If no solution is found from an abstraction level, we can always go back to the previous level by applying the complete refinement and search for the solution. In the worst case, we may end up searching for a solution in the original service space. Therefore, we never miss a solution, if one exists.

\section{Experimental Results}\label{sec:result}
\noindent
Our proposed algorithms were implemented in Java. We implemented our framework on top of two composition algorithms \cite{DBLPChattopadhyayB16} and \cite{xia2013web}. We first applied our framework on the WSC-2008 datasets \cite{bansal2008wsc}. Since the services in the WSC dataset are generated complete randomly \cite{wagner2011qos}, we did not find significant reduction in any abstraction step. For the $2^{nd}$ dataset, we found 1 pair of equivalent services, for the $7^{th}$ dataset 1 pair of equivalent services and 1 pair of dominant services and finally for the $8^{th}$ dataset, we found 2 pairs of dominant services. For rest of the datasets, we could not find any reduction due to abstraction. Hence, our framework could not provide significant performance gain for WSC-2008 datasets. Therefore, we demonstrate our method on a synthetically generated dataset. 

% \noindent
Our synthetic dataset is an extension of our example dataset as shown in Table \ref{tab:ourData}. Our example dataset is constructed from the dataset described in \cite{wagner2011qos}. We now discuss the construction of our synthetic dataset. We first generated a random number of concepts and assigned a random relationship (identical / sub / super / unrelated) between the concepts. We then randomly generated the number of input-output parameters and randomly associated them with the concepts. However, the services were also generated in a semi random manner. We increased the number of services in a random manner maintaining the relationship (equivalence / dominance / IIOE / unrelated) between the services as discussed above and accordingly assigned the input-output parameters, preconditions and postconditions. Finally, we generated the QoS parameters of the services. In this paper, we considered 4 QoS parameters: response time, throughput, reliability, availability and modeled them assuming they follow normal distribution \cite{c1}. Once the service repository was created, we generated a random query having random inputs, outputs, input-specifications, output-requirements and QoS constraints. We now analyze the results.

%%%%%%%%%%%%%%%%%%%%%%%%%%%%%%%%%%%%%%%%%%%%%%%%%%%%%%%%%%%%%%%%%%%%%%%%%%%%%%
\begin{table*}[htb]\centering
% \tbl{ }{
\caption{{Computation Time (ms) and speed up for dependency graph construction}}
\label{tab:time1}
%   \begin{center}
  \begin{tabular}{c||c|c|c|c||c|c|c}
  DataSet & Abs. 0 & Abs. 1 & Abs. 2 & Abs. 3 & $SU_1$ & $SU_2$ & $SU_3$\\
  \hline
  RDS1 & 731 & 70 & 56 & 28 & 10.44 & 1.25 & 2\\ 
  RDS2 & 3493 & 109 & 89 & 19 & 32.05 & 1.22 & 4.68\\ 
  RDS3 & 6159 & 380 & 221 & 154 & 16.21 & 1.72 & 1.44\\ 
  RDS4 & 17930 & 893 & 532 & 334 & 20.08 & 1.68 & 1.59\\ 
  RDS5 & 15904 & 874 & 421 & 343 & 18.2 & 2.08 & 1.23\\ 
   \end{tabular}
%  }
\end{table*}
%%%%%%%%%%%%%%%%%%%%%%%%%%%%%%%%%%%%%%%%%%%%%%%%%%%%%%%%%%%%%%%%%%%%%%%%%%%%%%
\begin{table*}[htb]\centering
 \caption{Computation Time (ns) and speed up for optimal RT computation}\label{tab:time2}
%   \begin{center}
  \begin{tabular}{c||c|c|c|c||c|c|c}
  DataSet & Abs. 0 & Abs. 1 & Abs. 2 & Abs. 3 & $SU_1$ & $SU_2$ & $SU_3$\\
  \hline
  RDS1 & 37523858 & 2111240 & 1921228 & 884105 & 17.77 & 1.1 & 2.17\\  
  RDS2 & 348775813 & 6014884 & 5817691 & 1079896 & 57.99 & 1.03 & 5.39\\ 
  RDS3 & 192147284 & 3180854 & 3021811 & 2012347 & 60.41 & 1.05 & 1.50\\
  RDS4 & 389731251 & 17119550 & 15407595 & 3599254 & 22.77 & 1.11 & 4.28\\  
  RDS5 & 209139820 & 4896525 & 4651698 & 3370275 & 42.71 & 1.05 & 1.38\\
   \end{tabular}
%  \end{center}
\end{table*}
%%%%%%%%%%%%%%%%%%%%%%%%%%%%%%%%%%%%%%%%%%%%%%%%%%%%%%%%%%%%%%%%%%%%%%%%%%%%%%
\begin{table*}[htb]\centering
 \caption{Computation Time (ns) and speed up for optimal throughput computation}\label{tab:time3}
%   \begin{center}
  \begin{tabular}{c||c|c|c|c||c|c|c}
  DataSet & Abs. 0 & Abs. 1 & Abs. 2 & Abs. 3 & $SU_1$ & $SU_2$ & $SU_3$\\
  \hline
  RDS1 & 37098104 & 2127783 & 1978838 & 1027102 & 17.44 & 1.075 & 1.93\\ 
  RDS2 & 425637217 & 3029113 & 2998821 & 1040155 & 140.52 & 1.01 & 2.88\\ 
  RDS3 & 281199735 & 3893887 & 3777070 & 2154407 & 72.22 & 1.03 & 1.75\\ 
  RDS4 & 38649038 & 7178013 & 7106232 & 2382924 & 5.38 & 1.01 & 2.98\\  
  RDS5 & 264711653 & 3455560 & 3386448 & 2291248 & 76.60 & 1.02 & 1.48\\ 
   \end{tabular}
%  \end{center}
\end{table*}
%%%%%%%%%%%%%%%%%%%%%%%%%%%%%%%%%%%%%%%%%%%%%%%%%%%%%%%%%%%%%%%%%%%%%%%%%%%%%%
%%%%%%%%%%%%%%%%%%%%%%%%%%%%%%%%%%%%%%%%%%%%%%%%%%%%%%%%%%%%%%%%%%%%%%%%%%%%%%%
\begin{table*}[htb]\centering
 \caption{Computation Time (ms) for constraint satisfying solution computation}\label{tab:result}
%   \begin{center}
  \begin{tabular}{c||c|c||c|c||c|c|c|c}
  \multirow{2}{*}{Level} & \multicolumn{2}{c||}{Exp. 1} & \multicolumn{2}{c||}{Exp. 2} & \multicolumn{4}{c}{Exp. 3}\\
   \cline{2-9}
   & \# Services & Time & \# Services & Time & \# Services & Time & Time & Time\\ \hline
   Abs. 0 & 3846 & 29862 & 9718 & 25896 & 98361 & - & - & -\\
   Abs. 1 & 1297 & 11596 & 6389 & 13791 & 53846 &  - & - & -\\
   Abs. 2 & 1123 & 3276 & 5231 & 11563 & 15926 &  3695 & - & -\\
   Abs. 3 & 201  & 9698 & 1226 & 2215 & 1856 & 1532 & 7197 & 23458\\
   \end{tabular}
%  \end{center}
\end{table*}
% \vspace{-0.5cm}
%%%%%%%%%%%%%%%%%%%%%%%%%%%%%%%%%%%%%%%%%%%%%%%%%%%%%%%%%%%%%%%%%%%%%%%%%%%%%%

We present our experiment in two parts. In the first part, we show the performance gain achieved by our framework at runtime, when implemented on top of \cite{xia2013web}. Figure \ref{fig:analysis1} shows the reduction in the number of services in the service repository by applying each abstraction level for 5 random datasets ($RDS1$ to $RDS5$).
% Figure \ref{fig:analysisRepository} shows the reduction in the number of services in the service repository by applying each abstraction level for 5 random datasets ($RDS1$ to $RDS5$). On an average, there are 3 times, 2 times and 0 times (no reduction) reductions in the number of services in the service repository respectively. It may noted, in the last level we have not found any reduction in the number of services, as we have claimed in the third abstraction level. 
Figure \ref{fig:analysis2} shows the reduction in the number of services in the dependency graph constructed to respond to a query by applying each abstraction level for 5 random datasets. On an average, there are 6x, 10x and 16x reductions in the number of services in the dependency graph for abstraction levels 1, 2 and 3 respectively with respect to the number of unabstracted services. Table \ref{tab:time1} shows the computation time to construct the dependency graph generated in response to the query for each of the 3 abstractions. Columns $2$ to $5$ of the table show the computation time required by the underlying algorithm for 4 levels of services (the $0^{th}$ level corresponds to without abstraction). Columns $6$ to $8$ of the table show the speed up ($SU_i$) for $i^{th}$ ($i = 1, 2, 3$) abstraction level respectively with respect to the un-abstracted service space (i.e., Level 0). As evident from the table, we achieved, on an average, 19x, 30x, 70x  speed up for abstraction levels 1, 2 and 3 respectively. Figure \ref{fig:analysisRT} shows the response time (RT) obtained by Algorithm \cite{xia2013web} for 4 levels of services. The response time obtained by our framework, on an average, degrades 4x, 5x, 6x for  abstraction levels 1, 2 and 3 respectively, while the computation speed up achieved by our framework to compute response time is on an average 40x, 43x, 126x (shown in Table \ref{tab:time2}). Similarly, the throughput obtained by our framework, on an average, degrades 1.28x, 1.33x, 1.36x respectively for abstraction levels 1, 2 and 3 (shown in Figure \ref{fig:analysisTR}), while the speed up achieved by our framework to compute throughput is on an average 62x, 64x, 142x (as in Table \ref{tab:time3}).

% \vspace{-0.1cm}

\begin{figure}
 \centering
 \includegraphics[width=\linewidth]{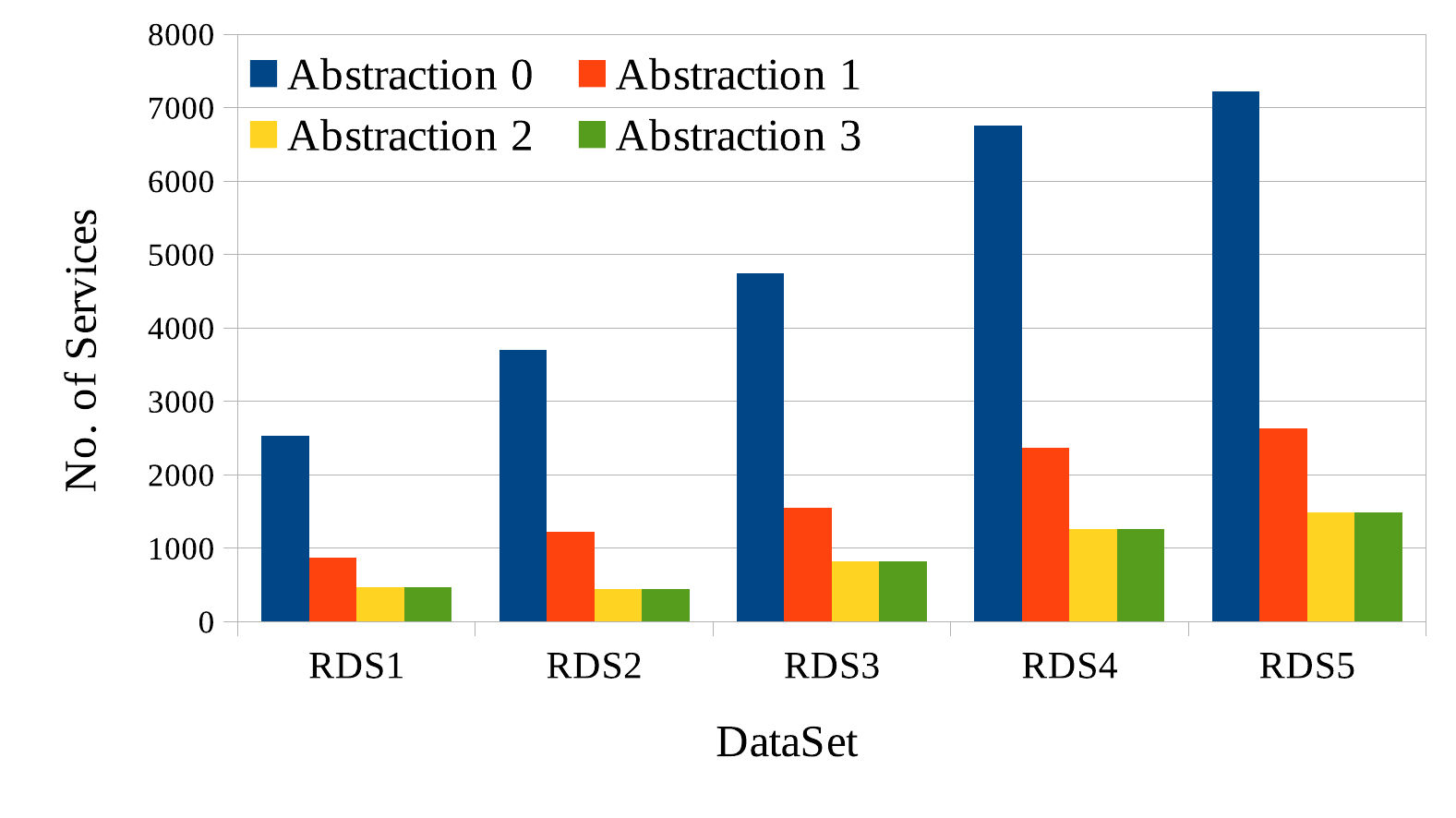}
 \caption{{No. of services in service repository}}
\label{fig:analysis1}
\end{figure}

\begin{figure}
 \centering
 \includegraphics[width=\linewidth]{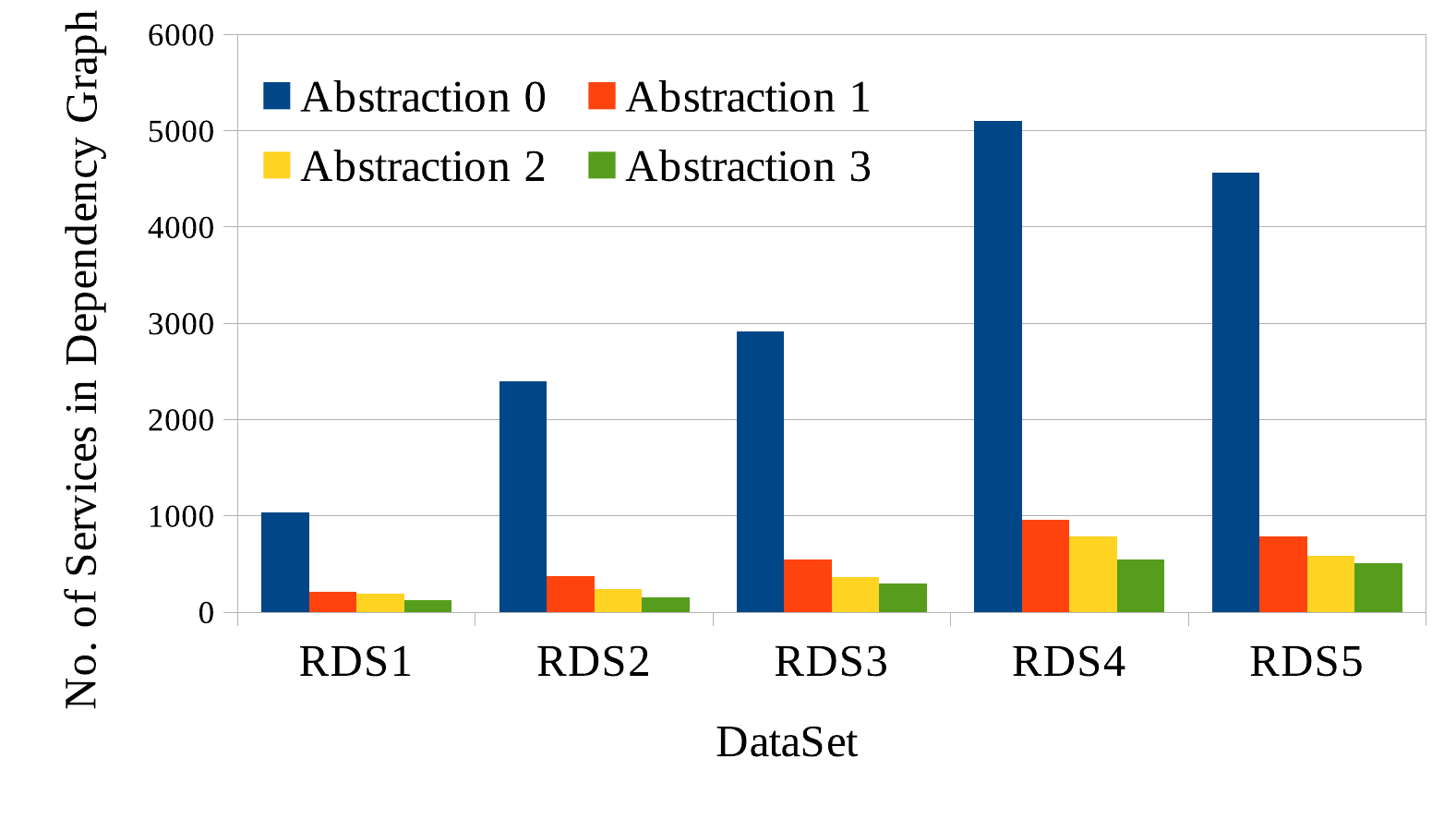}
 \caption{{No. of services in dependency graph}}
 \label{fig:analysis2}
\end{figure}

%  \begin{figure}
%  \centering
%  (a)\includegraphics[width=0.4\linewidth]{RT-eps-converted-to.pdf}
%  (b)\includegraphics[width=0.4\linewidth]{TR-eps-converted-to.pdf}
%  \caption{{(a) Response Time (b) Throughput}}
%  \label{fig:analysis2}
%  \end{figure}

 \begin{figure}
 \centering
 \includegraphics[width=\linewidth]{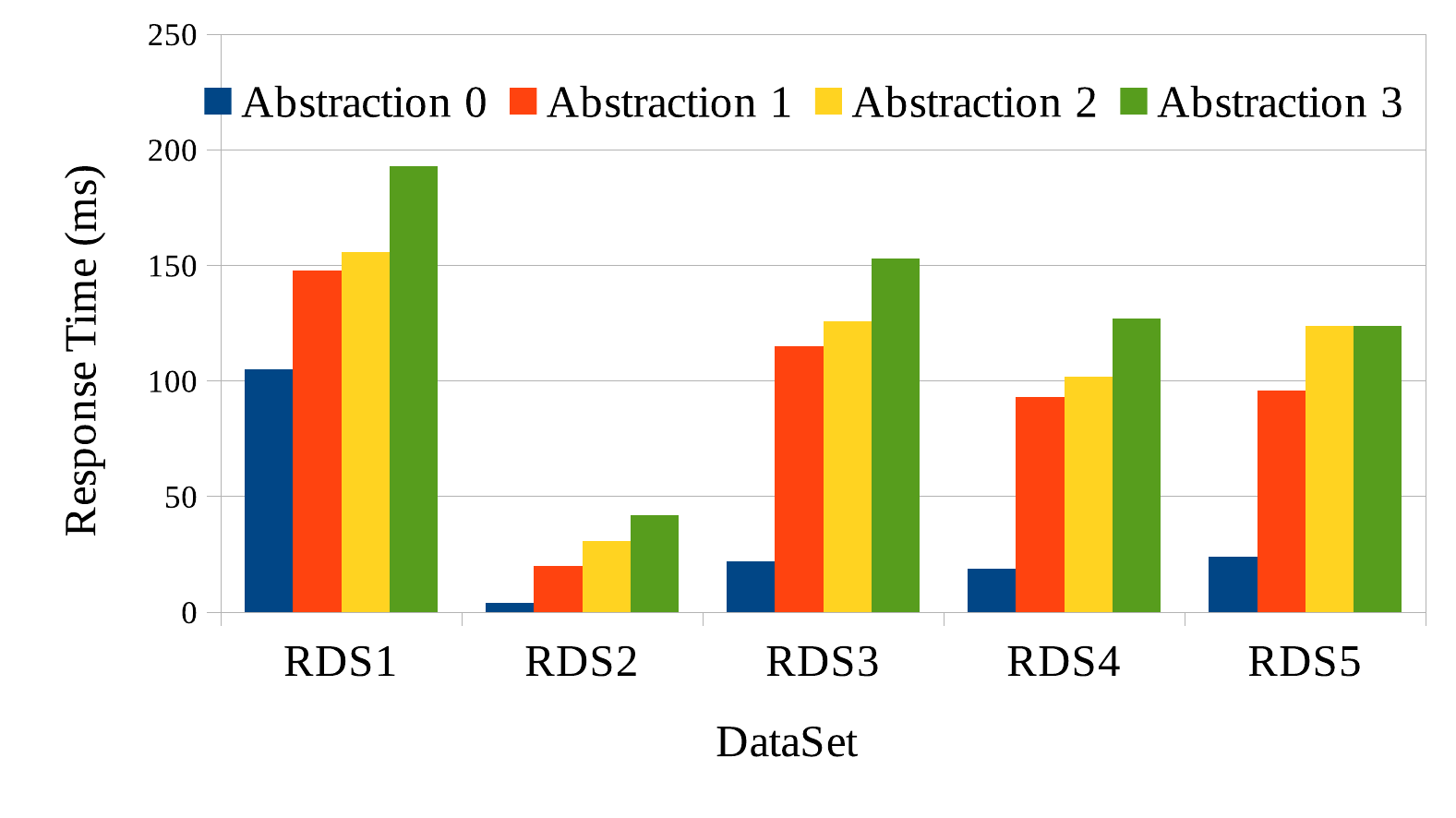}
 \caption{{Response Time}}
 \label{fig:analysisRT}
 \end{figure}
 
 \begin{figure}
 \centering
 \includegraphics[width=\linewidth]{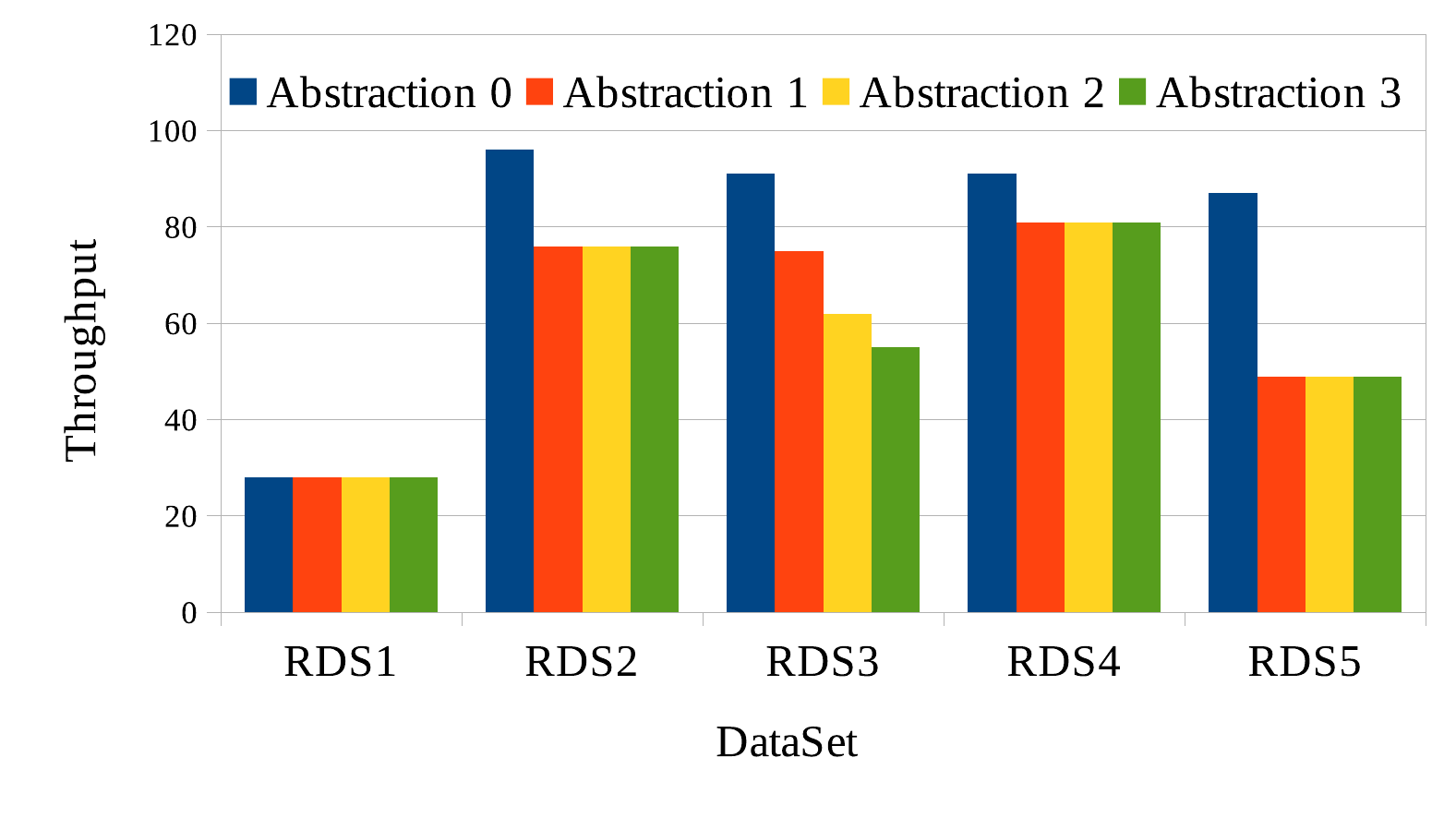}
 \caption{{Throughput}}
 \label{fig:analysisTR}
 \end{figure}

% \begin{table}[htb]
%  \tiny
%  \caption{{Result on WSC-2009 benchmark}}
%   \begin{center}
%   \begin{tabular}{c|c|c|c|c}
%   Dataset & without abstraction & after $1^{st}$ level & after $2^{nd}$ level & after $3^{rd}$ level \\
%   \hline
%   WSC01 & 572   & 572   & 570   & 570\\
%   WSC02 & 4129  & 4129  & 4123  & 4123\\
%   WSC03 & 8138  & 8138  & 8137  & 8137\\
%   WSC04 & 8301  & 8301  & 8301  & 8301\\
%   WSC05 & 15211 & 15211 & 15201 & 15201\\
%   \end{tabular}\label{tab:WSC}
%  \end{center}
% \end{table}
% \input{caseStudy}

\iffalse
\begin{figure}
\centering
\includegraphics[width=0.6\linewidth]{noOfService-eps-converted-to.pdf}\vspace{-0.2cm}
\caption{{No. of services in service repository}}
\label{fig:analysisRepository}
\end{figure}
\fi
In the second part of our experiment, we show the performance of our framework implemented on top of \cite{DBLPChattopadhyayB16}.
Here, we considered the query having QoS constraints. We conducted three sets of experiments (Exp. 1, Exp. 2 and Exp. 3 of Table \ref{tab:result}) varying the number of services and query and recorded the computation time to serve a query. In the final experiment (Exp. 3 of Table \ref{tab:result}), we varied the constraints of the QoS parameters ($9^{th}$ Column has tighter constraints than $8^{th}$, which has tighter constraints than $7^{th}$) for the same query and recorded the result. As can be seen from the table, in all 3 experiments, as the abstraction level increases, the number of services either remains same or decreases, which we have claimed in this paper. However, it is not guaranteed that the computation time to find a solution always improves as the abstraction level increases, because in some cases, complete refinement is necessary (as can be seen in the first experiment). In the first experiment, we have shown as the abstraction level increases, the time requirement to compute a solution decreases upto the { $2^{nd}$} level of abstraction. However, in the { $3^{rd}$} abstraction level, the time requirement increases. It is due to the fact that in the { $3^{rd}$} level of abstraction no solution could be found. Therefore, complete refinement was performed. In the last experiment, ``-" indicates no solution is generated due to time out. It may be noted, no solution was generated in the lower level of abstraction as well. However, at the higher level, we obtained a solution. It may be observed, on an average, the composition algorithm in the abstract service space outperforms the composition algorithm in the un-abstract service space.

% \noindent

\section{Related Work}\label{related}
\noindent
Research on service composition has been carried out in multiple directions. A significant amount of work considering semantic web service composition has been performed considering optimality being the primary concern (e.g., graph based approaches \cite{chen2012redundant,rodriguez2015hybrid}, ILP based approaches \cite{DBLPChattopadhyayB16,schuller2012cost}, AI planning based approaches \cite{oh2008effective,mostafa2015multi}). However, in real time,  optimality has proved to be an expensive requirement \cite{lecue2009towards,chattopadhyay2015scalable} for service composition, since most of the optimal methods do not scale for large service repositories. Therefore, a number of heuristic methods \cite{wuqos,6009375,guidara2015heuristic,el2010tqos} have been proposed, where computation time is considered as the main objective.
In \cite{Alrifai2010} and \cite{alrifai2009combining}, authors have proposed a multi constrained QoS aware service composition approach, instead of finding the optimal solutions. Dynamic binding is the main concern of \cite{alrifai2009combining}, where authors propose to generate the skyline services for each task and cluster the services using the K-means algorithm. In \cite{Alrifai2010}, authors proposed an Integer Linear Programming (ILP) based approach, where ILP is used to divide the global constraints into a set of local constraints and then using the local constraints, the service selection is done for each task. In \cite{schuller2012cost}, authors proposed ILP based methods to solve multi-constrained service composition. A significant amount of work has been done based on evolutionary  algorithms, such as Particle Swarm Optimization \cite{liao2013multi}, Ant Colony Optimization \cite{shanshan2012improved}, Bee Colony Optimization \cite{liu2014parameter}, Genetic Algorithms \cite{zhang2013genetic,wang2008optimal}, NSGA2 \cite{hashmi2013automated,wagner2012multi}. 
Though these algorithms can rapidly generate solutions and handle large and complex service spaces \cite{lecue2009towards}, they compromise on the solution quality \cite{song2011workflow,pistore2005automated}. In \cite{yan2015anytime}, the authors proposed a planning graph based approach and an anytime algorithm \cite{yan2015anytime}
that attempts to maximize the utility. In \cite{Zeng:2003:QDW:775152.775211}, authors proposed an ILP based method 
to maximize the utility. 
In \cite{DBLPChattopadhyayB16}, the authors proposed ILP based multi-constrained service composition.
In \cite{cao2007service}, the authors analyze the relation between multi-objective service composition and the
Multi-choice, Multi-dimension 0-1 Knapsack Problem (MMKP).
There are few methods \cite{lecue2009towards,chattopadhyay2015scalable} in literature, that enlighten the issue of search space reduction of the composition algorithm. However, they often fail to generate a solution with desired quality. 

In contrast to existing literature, we propose a framework for semantic web service composition based on abstraction refinement that aims to expedite the solution construction time by reducing the search space without compromising on the solution quality. Service clustering based on syntactic similarity in terms of input-output parameters has been dealt with in our earlier work \cite{7933195}. However, the notion of semantics of a web service has not been dealt with in that work. In this paper, we bring the notion of semantics and propose a methodology that provides a scalable way of pruning the composition search space. Our method can be applied on top of any service composition algorithm to improve its performance. To conclude, we do not propose any new service composition algorithm, rather a framework on top of existing algorithms. This adds a unique novelty to this work.

\section{Conclusion}
\label{sec:conclusion}
\noindent
This paper presents an  approach for reducing the search space for semantic web service composition based on abstraction refinement. For a large dataset, the approach can be quite efficient, since it greatly improves 
performance. As future work, we are currently working on designing more sophisticated refinement techniques. We also aim to perform more extensive experiments of this framework on real datasets. 

\scriptsize
\bibliographystyle{IEEEtran}
\bibliography{ref}
\end{document}